\documentclass[11pt]{article}
\usepackage{fullpage,graphicx,psfrag,amsmath,amsfonts,verbatim,bm}
\usepackage[table,xcdraw]{xcolor}

\usepackage{amsmath,amsfonts,bm}

\newcommand{\Vp}{V_h^{\pi_k}}
\newcommand{\hVp}{\hat{V}_h^{\pi_k}}
\newcommand{\Vpn}{V_{h+1}^{\pi_k}}
\newcommand{\hVpn}{\hat{V}_{h+1}^{\pi_k}}
\newcommand{\Qp}{Q_h^{\pi_k}}
\newcommand{\hQp}{\hat{Q}_h^{\pi_k}}
\newcommand{\sPsa}{\sigma_{\mathcal{P}_h(s,a)}}
\newcommand{\hPsa}{\sigma_{\hat{\mathcal{P}}_h(s,a)}}
\newcommand{\sPs}{\sigma_{\mathcal{P}_h(s)}}
\newcommand{\hPs}{\sigma_{\hat{\mathcal{P}}_h(s)}}

\newcommand{\ie}{{\it i.e.}}

\def\approxcorrect{\checkmark\kern-1.1ex\raisebox{.89ex}{$\times$}}

\def\eqref#1{equation~\ref{#1}}

\def\1{\bm{1}}

\DeclareMathAlphabet{\mathsfit}{\encodingdefault}{\sfdefault}{m}{sl}
\SetMathAlphabet{\mathsfit}{bold}{\encodingdefault}{\sfdefault}{bx}{n}

\def\gA{{\mathcal{A}}}

\def\gG{{\mathcal{G}}}

\def\gM{{\mathcal{M}}}

\def\gP{{\mathcal{P}}}

\def\gS{{\mathcal{S}}}

\usepackage[small,bf]{caption}
\usepackage{natbib}
\usepackage{amsthm,apxproof}
\usepackage{thmtools,amsmath,amsfonts} \usepackage{amssymb}
\usepackage{pifont}
\usepackage{thm-restate}
\usepackage{booktabs}

\usepackage{algorithmic,algorithm}
\usepackage{caption}
\usepackage{subcaption}
\usepackage{wrapfig}
\usepackage{mathtools}
\usepackage{multirow}
\usepackage{hyperref}
\usepackage{url}
\usepackage{booktabs}
\newtheorem{thm}{Theorem}
\newtheorem{lem}{Lemma}

\newtheorem{defn}{Definition}[section]

\newtheorem{rem}{Remark}[section]

\allowdisplaybreaks
\let\svthefootnote\thefootnote
\newcommand\freefootnote[1]{%
  \let\thefootnote\relax%
  \footnotetext{#1}%
  \let\thefootnote\svthefootnote%
}
\title{Online Policy Optimization for Robust MDP}
\author{
  Jing Dong \thanks{The Chinese University of Hong Kong, Shenzhen}\\
  \texttt{jingdong@link.cuhk.edu.cn}
  \and
  Jingwei Li \thanks{Tsinghua University}\\
  \texttt{ljw22@mails.tsinghua.edu.cn}
  \and
  Baoxiang Wang \footnotemark[1] \\
  \texttt{bxiangwang@cuhk.edu.cn}
  \and
  Jingzhao Zhang \footnotemark[2]\\
  \texttt{jingzhaoz@mail.tsinghua.edu.cn}
}
\date{}
\begin{document} \freefootnote{Authors are listed in alphabetical order.}
\maketitle

\begin{abstract}
Reinforcement learning (RL) has exceeded human performance in many synthetic settings such as video games and Go. However, real-world deployment of end-to-end RL models is less common, as RL models can be very sensitive to slight perturbation of the environment. 
The robust Markov decision process (MDP) framework---in which the transition probabilities belong to an uncertainty set around a nominal model---provides one way to develop robust models. 
While previous analysis shows RL algorithms are effective assuming access to a generative model, it remains unclear whether RL can be efficient under a more realistic online setting, which requires a careful balance between exploration and exploitation. 
In this work, we consider online robust MDP by interacting with an unknown nominal system. We propose a robust optimistic policy optimization algorithm that is provably efficient. 
To address the additional uncertainty caused by an adversarial environment, our model features a new optimistic update rule derived via Fenchel conjugates. 
Our analysis establishes the first regret bound for online robust MDPs. 
\end{abstract}

\section{Introduction}
The rapid progress of reinforcement learning (RL) algorithms enables trained agents to navigate around complicated environments and solve complex tasks. The standard reinforcement learning methods, however, may fail catastrophically in another environment, even if the two environments only differ slightly in dynamics \citep{farebrother2018generalization,packer2018assessing,cobbe2019quantifying,song2019observational,raileanu2021decoupling}. In practical applications, such mismatch of environment dynamics are common and can be caused by a number of reasons, e.g., model deviation due to incomplete data, unexpected perturbation and possible adversarial attacks. Part of the sensitivity of standard RL algorithms stems from the formulation of the underlying Markov decision process (MDP). In a sequence of interactions, MDP assumes the dynamic to be unchanged, and the trained agent to be tested on the same dynamic thereafter. 

To model the potential mismatch between system dynamics, the framework of robust MDP is introduced to account for the uncertainty of the parameters of the MDP \citep{satia1973markovian,white1994markov,nilim2005robust,iyengar2005robust}. Under this framework, the dynamic of an MDP is no longer fixed but can come from some uncertainty set, such as the rectangular uncertainty set, centered around a nominal transition kernel. The agent sequentially interacts with the nominal transition kernel to learn a policy, which is then evaluated on the worst possible transition from the uncertainty set. Therefore, instead of searching for a policy that may only perform well on the nominal transition kernel, the objective is to find the worst-case best-performing policy. This can be viewed as a dynamical zero-sum game, where the RL agent tries to choose the best policy while nature imposes the worst possible dynamics. Intrinsically, solving the robust MDPs involves solving a max-min problem, which is known to be challenging for efficient algorithm designs.

More specifically, if a generative model (also known as a simulator) of the environment or a suitable offline dataset is available, one could obtain a $\epsilon$-optimal robust policy with $\Tilde{O}(\epsilon^{-2})$ samples under a rectangular uncertainty set \citep{qi2020robust,panaganti2022sample,wang2022policy,ma22distribution}. Yet the presence of a generative model is stringent to fulfill for real applications. In a more practical online setting, the agent sequentially interacts with the environment and tackles the exploration-exploitation challenge as it balances between exploring the state space and exploiting the high-reward actions. 
In the robust MDP setting, previous sample complexity results cannot directly imply a sublinear regret in general \citet{dann2017unifying} and so far no asymptotic result is available. A natural question then arises:
\begin{center}
    \textit{Can we design a robust RL algorithm that attains sublinear regret under robust MDP with rectangular uncertainty set?}
\end{center}

In this paper, we answer the above question affirmatively and propose the first policy optimization algorithm for robust MDP under a rectangular uncertainty set. One of the challenges for deriving a regret guarantee for robust MDP stems from its adversarial nature. As the transition dynamic can be picked adversarially from a predefined set, the optimal policy is in general randomized \citep{wiesemann2013robust}. This is in contrast with conventional MDPs, where there always exists a deterministic optimal policy, which can be found with value-based methods and a greedy policy (e.g. UCB-VI algorithms). Bearing this observation, we resort to policy optimization (PO)-based methods, which directly optimize a stochastic policy in an incremental way.

With a stochastic policy, our algorithm explores robust MDPs in an optimistic manner. To achieve this robustly, we propose a carefully designed bonus function via the dual conjugate of the robust bellman equation. This quantifies both the uncertainty stemming from the limited historical data and the uncertainty of the MDP dynamic. In the episodic setting of robust MDPs, we show that our algorithm attains sublinear regret $O(\sqrt{K})$ for both $(s,a)$ and $s$-rectangular uncertainty set, where $K$ is the number of episodes. In the case where the uncertainty set contains only the nominal transition model, our results recover the previous regret upper bound of non-robust policy optimization \citep{shani2020optimistic}. Our result achieves the first provably efficient regret bound in the online robust MDP problem, as shown in Table~\ref{table:compare}. We further validated our algorithm with experiments.



\begin{table}[h]
\centering\caption{Comparisons of previous results and our results, where $S,A$ are the size of the state space and action space, $H$ is the length of the horizon, $K$ is the number of episodes, $\rho$ is the radius of the uncertainty set and $\epsilon$ is the level of suboptimality. We shorthand $\iota = \log(SAH^2 K^{3/2} (1 + \rho))$. The regret upper bound by \citet{panaganti2022sample} are obtained through converting their sample complexity results and the sample complexity result for our work is converted through our regret bound. We use ``GM'' to denote the requirement of a generative model and ``for PE'' to denote that the result is only for robust policy evaluation (estimating a robust value function for a fixed policy). The reference to the previous works are [A]: \citet{panaganti2022sample}, [B]: \citet{wang2021online}, [C]: \citet{badrinath2021robust}, [D]: \citet{yang2021towards}.}\label{table:compare}
\begin{tabular}{|c|c|c|c|c|c|}
\hline
                               & Algorithm                                                                   & Requires  & Rectangular        & Regret                                                                                                                                            & Sample Complexity                                                                                                                  \\ \hline
{[}A{]}                        & \begin{tabular}[c]{@{}c@{}}Value\\ based\end{tabular}                   & GM            & $(s,a)$      & \begin{tabular}[c]{@{}c@{}}NA\\ \end{tabular}                                                 & \begin{tabular}[c]{@{}c@{}}$O\left(\frac{H^4S^2A}{\epsilon^2} \right)$\\ \end{tabular}                           \\ \hline

{[}B{]}                        & \begin{tabular}[c]{@{}c@{}}Value\\ based\end{tabular}                   & -            & $(s,a)$      & \begin{tabular}[c]{@{}c@{}}NA\\ \end{tabular}                                                 & Asymptotic                \\ \hline

{[}C{]}                        & \begin{tabular}[c]{@{}c@{}}Policy\\ based\end{tabular}                   & -            & $(s,a)$      & \begin{tabular}[c]{@{}c@{}}NA\\ \end{tabular}                                                 & Asymptotic                \\ \hline
\multirow{2}{*}{{[}D{]}}       & \multirow{2}{*}{\begin{tabular}[c]{@{}c@{}}Value\\ based\end{tabular}}  & \multirow{2}{*}{GM} & $(s,a)$                                                                                  & NA                                                             & \begin{tabular}[c]{@{}c@{}}$\Tilde{O}\left(\frac{H^2S^2A(2 + \rho)^2}{\rho^2\epsilon^2} \right)$  for PE\end{tabular} \\ \cline{4-6} &                                                                         &                     & $s$& NA                                             & \begin{tabular}[c]{@{}c@{}}$\Tilde{O}\left(\frac{H^2S^2A^2(2 + \rho)^2}{\rho^2\epsilon^2} \right)$ for PE\end{tabular}   
                               \\ \hline

\multirow{2}{*}{\textbf{Ours}} & \multirow{2}{*}{\begin{tabular}[c]{@{}c@{}}Policy\\ based\end{tabular}} & \multirow{2}{*}{-}  & $(s,a)$ &  $O \left( 
SH^2  \sqrt{AK\iota}\right)$ & $O\left( \frac{H^4 S^2 A \iota}{\epsilon^2}\right)$                                                  

\\ \cline{4-6} 
                               &                                                                         & & $s$ & $O \left( SA^2 H^2\sqrt{K\iota}\right) $                      & $O \left( \frac{H^4S^2 A^4  \iota}{\epsilon^2}\right)$                                                                                                                            \\ \hline
\end{tabular}
\end{table} 

\section{Related work}
\paragraph{RL with robust MDP} Different from conventional MDPs, robust MDPs allow the transition kernel to take values from an uncertainty set. The objective in robust MDPs is to learn an optimal robust policy that maximizes the worst-case value function. When the exact uncertainty set is known, this can be solved through dynamic programming methods \citep{iyengar2005robust,nilim2005robust,mannor2012lightning}. Yet knowing the exact uncertainty set is a rather stringent requirement for most real applications. 
If one has access to a generative model, several model-based reinforcement learning methods are proven to be statistically efficient. With the different characterization of the uncertainty set, these methods can enjoy a sample complexity of $O(1/\epsilon^2)$ for an $\epsilon$-optimal robust value function \citep{panaganti2022sample,yang2021towards}. Similar results can also be achieved if an offline dataset is present, for which previous works \citet{qi2020robust,zhou2021finite,kallus2022doubly,ma22distribution} show the $O(1/\epsilon^2)$ sample complexity for an $\epsilon$-optimal policy. 

In the case of online RL, the only results available are asymptotic. In the case of discounted MDPs, \citet{wang2021online,badrinath2021robust} study the policy gradient method and show an $O(\epsilon^{-3})$ convergence rate for an alternative learning objective (a smoothed variant), which could be equivalent to the original policy gradient objective in an asymptotic regime.
These results in sample complexity and asymptotic regimes in general cannot imply sublinear regret in robust MDPs \citep{dann2017unifying}. 

\paragraph{RL with adversarial MDP}
Another line of works characterizes the uncertainty of the environment through the adversarial MDP formulation, where the environmental parameters can be adversarially chosen without restrictions. 
This problem is proved to be NP-hard to obtain a low regret \citep{even2004experts}.
Several works study the variant where the adversarial could only modify the reward function, while the transition dynamics of the MDP remain unchanged.
In this case, it is possible to obtain policy-based algorithms that are efficient with a sublinear regret \citep{rosenberg2019online,jin2020simultaneously,pmlr-v119-jin20c,shani2020optimistic,cai2020provably}.
On a separate vein, it investigates the setting where the transition is only allowed to be adversarially chosen for $C$ out of the $K$ total episodes. A regret of $O(C^2 + \sqrt{K})$ are established thereafter \citep{lykouris2021corruption,chen2021improved,zhang2022corruption}.


\paragraph{Non-robust policy optimization}
The problem of policy optimization has been extensively investigated under non-robust MDPs \citep{neu2010online,cai2020provably,shani2020optimistic,wu2022nearly,chen2021minimax}. The proposed methods are proved to achieve sublinear regret.
The methods are also closely related to empirically successful policy optimization algorithms in RL, such as PPO \citet{schulman2017proximal} and TRPO \citet{schulman2015trust}. 

\section{Robust MDP and uncertainty sets}
In this section, we describe the formal setup of robust MDP. We start with defining some notations.
\paragraph{Robust Markov decision process}
We consider an episodic finite horizon robust MDP, which can denoted by a tuple $\gM = \langle \gS, \gA, H, $ $\{\gP\}_{h=1}^H, \{r\}_{h=1}^H \rangle$. Here $\gS$ is the state space, $\gA$ is the action space, $\{r\}_{h=1}^H$ is the time-dependent reward function, and $H$ is the length of each episode. Instead of a fixed step of time-dependent uncertainty kernels, the transitions of the robust MDP is governed by kernels that are within a time-dependent uncertainty set $\{\gP\}_{h=1}^H$, $\ie$, time-dependent transition $P_h \in \gP_h \subseteq \Delta_{\gS}$ at time $h$. 

The uncertainty set $\gP$ is constructed around a nominal transition kernel $P_h = \{P_h^o\}$, and all transition dynamics within the set are close to the nominal kernel with a distance metric of one's choice. Different from an episodic finite-horizon non-robust MDP, the transition kernel $P$ may not only be time-dependent but may also be chosen (even adversarially) from a specified time-dependent uncertainty set $\gP$. We consider the case where the rewards are stochastic. This is, on state-action $(s,a)$ at time $h$, the immediate reward is $R_h(s,a) \in [0,1]$, which is drawn i.i.d from a distribution with expectation $r_h(s,a)$.
With the described setup of robust MDPs, we now define the policy and its associated value.

\paragraph{Policy and robust value function}
A time-dependent policy $\pi$ is defined as $\pi = \{\pi_h\}_{h=1}^H$, where each $\pi_h$ is a function from $\gS$ to the probability simplex over actions, $\Delta(\gA)$. 
If the transition kernel is fixed to be $P$, the performance of a policy $\pi$ starting from state $s$ at time $h$ can be measured by its value function, which is defined as 
\begin{align*}
    V_h^{\pi, P}(s) = \mathbb{E}_{\pi, P}\left[ \sum^H_{h^\prime = h} r_{h^\prime}(s_{h^\prime}, a_{h^\prime}) \mid s_h = s\right] \,.
\end{align*}
In robust MDP, the robust value function instead measures the performance of $\pi$ under the worst possible choice of transition $P$ within the uncertainty set. Specifically, the value and the Q-value function of a policy given the state action pair $(s,a)$ at step $h$ are defined as 
\begin{align*}
    V^{\pi}_h (s) = \ & \min_{\{P_h\} \in \{\gP_h\}} V^{\pi, \{P\}}_h (s) \,, \nonumber \\
    Q^{\pi}_h (s,a) = \ & \min_{\{P_h\} \in \{\gP_h\}} \mathbb{E}_{\pi, \{P\}}\left[\sum^H_{h^\prime =h} r_h (s_{h^\prime} , a_{h^\prime} ) \mid (s_h,a_h) = (s, a)\right] \,.
\end{align*}
The optimal value function is defined to be the best possible value attained by a policy
\begin{align*}
    V^{\ast}_h (s) = \max_{\pi} V^{\pi}_h (s) = \max_{\pi} \min_{\{P_h\} \in \{\gP_h\}} V^{\pi, \{P\}}_h (s) \,.
\end{align*}
The optimal policy is then defined to be the policy that attains the optimal value.

\paragraph{Robust Bellman equation}
Similar to non-robust MDP, robust MDP has the following robust bellman equation, which characterizes a relation to the robust value function.
\begin{align*}
    Q^{\pi}_h (s,a) = r(s,a) + \sigma_{\gP_h}(V_{h+1}^\pi)(s,a)\,, \quad V^{\pi}_h (s) = \langle Q^{\pi}_h (s,\cdot), \pi_h(\cdot, s) \rangle \,,
\end{align*} where
\begin{align}\label{eq:sigma}
    \sigma_{\gP_h}(V_{h+1}^\pi)(s,a) = \min_{P_h \in \gP_h} \limits P_h(\cdot \mid s,a) V_{h+1}^\pi \,, \quad P_h(\cdot \mid s,a) V = \sum_{s^\prime \in \gS} \limits P_h(s^\prime \mid s,a) V(s^\prime)\,.
\end{align}

Without additional assumptions on the uncertainty set, the optimal policy and value of the robust MDP are in general NP-hard to solve \citep{wiesemann2013robust}. One of the most commonly assumptions that make solving optimal value feasible is the rectangular assumption \citep{iyengar2005robust,wiesemann2013robust,badrinath2021robust,yang2021towards,panaganti2022sample}. 
\paragraph{Rectangular uncertainty sets}
To limit the level of perturbations, we assume that the transition kernels is close to the nominal transition measured via $\ell_1$ distance. We consider two cases.

The $(s,a)$-rectangular assumption assumes that the uncertain transition kernel within the set takes value independently for each $(s,a)$. We further use $\ell_1$ distance to characterize the $(s,a)$-rectangular set around a nominal kernel with a specified level of uncertainty.
\begin{defn}[$(s,a)$-rectangular uncertainty set \citet{iyengar2005robust,wiesemann2013robust}]\label{def:sa}
For all time step $h$ and with a given state-action pair $(s,a)$, the $(s,a)$-rectangular uncertainty set $\gP_h(s,a)$ is defined as 
\[
\gP_h(s,a) = \left\{\left\|P_h(\cdot \mid s,a) - P_h^o(\cdot \mid s,a) \right\|_1 \leq \rho,P_h(\cdot \mid s,a) \in \Delta(\gS) \right\} \,,
\]
where $P_h^o$ is the nominal transition kernel at $h$, $P_h^o(\cdot \mid s,a) > 0, \forall (s,a) \in \gS \times \gA$, $\rho$ is the level of uncertainty.
\end{defn}
With the $(s,a)$-rectangular set, it is shown that there always exists an optimal policy that is deterministic \cite{wiesemann2013robust}. 

One way to relax the $(s,a)$-rectangular assumption is to instead let the uncertain transition kernels within the set take value independent for each $s$ only. This characterization is then more general and its solution gives a stronger robustness guarantee. 
\begin{defn}[$s$-rectangular uncertainty set \citet{wiesemann2013robust}]\label{def:s}
For all time step $h$ and with a given state $s$, the $s$-rectangular uncertainty set $\gP_h(s)$ is defined as 
\[
\gP_h(s) = \left\{ \sum_{a \in \gA}\left\|P_h(\cdot \mid s,a) - P_h^o(\cdot \mid s,a) \right\|_1 \leq A \rho, P_h(\cdot \mid s,\cdot) \in \Delta(\gS)^{\gA}  \right\} \,,
\]
where $P_h^o$ is the nominal transition kernel at $h$, $P_h^o(\cdot \mid s,a) > 0, \forall (s,a) \in \gS \times \gA$, $\rho$ is the level of uncertainty.
\end{defn}
Different from the $(s,a)$-rectangular assumption, which guarantees the existence of a deterministic optimal policy, the optimal policy under $s$-rectangular set may need to be randomized \citep{wiesemann2013robust}. We also remark that the requirement of $P_h^o(\cdot \mid s,a) > 0$ is mostly for technical convenience.

Equipped with the characterization of the uncertainty set, we now describe the learning protocols and the definition of regret under the robust MDP. 

\paragraph{Learning protocols and regret}
We consider a learning agent repeatedly interacts with the environment in an episodic manner, over $K$ episodes. 
At the start of each episode, the learning agent picks a policy $\pi_k$ and interacts with the environment while executing $\pi_k$. 
Without loss of generality, we assume the agents always start from a fixed initial state $s$. The performance of the learning agent is measured by the cumulative regret incurred over the $K$ episodes. Under the robust MDP, the cumulative regret is defined to be the cumulative difference between the robust value of $\pi_k$ and the robust value of the optimal policy, 
\begin{align*}
    \text{Regret}(K) = \sum^K_{k=1} V_1^{\ast}(s_0) - V_1^{\pi_k} (s_0)\,,
\end{align*}
where $s_0^k$ is the initial state.

We highlight that the transition of the states in the learning process is specified by the nominal transition kernel $\{P_h^o\}_{h=1}^H$, though the agent only has access to the nominal kernel in an online manner. We remark that if the agent is asked to interact with a potentially adversarially chosen transition, the learning problem is NP-hard \cite{even2004experts}. 

One practical motivation for this formulation could be as follows. The policy provider only sees feedback from the nominal system, yet she aims to minimize the regret for clients who refuse to share additional deployment details for privacy purposes.

\section{Algorithm}
Before we introduce our algorithm, we first illustrate the importance of taking uncertainty into consideration. With the robust MDP, one of the most naive methods is to directly train a policy with the nominal transition model. However, the following proposition shows an optimal policy under the nominal policy can be arbitrarily bad in the worst-case transition (even worse than a random policy).  
\begin{restatable}[Suboptimality of non-robust optimal policy]{claim}{hard}\label{prop:hard}
   There exists a robust MDP $\gM = \langle \gS, \gA, \gP, r, H \rangle$ with uncertainty set $\gP$ of uncertainty radius $\rho$, such that the non-robust optimal policy is $\Omega(1)$-suboptimal to the uniformly random policy. 
 \end{restatable}
The proof of Proposition \ref{prop:hard} is deferred to Appendix \ref{appendix:prop}. With the above-stated result, it implies the policy obtained with non-robust RL algorithms, can have arbitrarily bad performance when the dynamic mismatch from the nominal transition. Therefore, we present the following robust optimistic policy optimization \ref{alg} to avoid this undesired result.   

\subsection{Robust optimistic policy optimization}
With the presence of the uncertainty set, the optimal policies may be all randomized \citep{wiesemann2013robust}. In such cases, value-based methods may be insufficient as they usually rely on a deterministic policy. We thus resort to optimistic policy optimization methods~\cite{shani2020optimistic}, which directly learn a stochastic policy. 

Our algorithm performs policy optimization with empirical estimates and encourages exploration by adding a bonus to less explored states. However, we need to propose a new efficiently computable bonus that is robust to adversarial transitions. We achieve this via solving a sub-optimization problem derived from Fenchel conjugate. We present Robust Optimistic Policy Optimization (ROPO) in Algorithm \ref{alg} and elaborate on its design components.

To start, as our algorithm has no access to the actual reward and transition function, we use the following empirical estimator of the transition and reward:
\begin{align}\label{eq:empirical}
    \hat{r}_h^k(s,a) =& \frac{\sum^{k - 1}_{k^\prime = 1} R_h^{k^\prime}(s,a)\mathbb{I} \left\{s_h^{k^\prime} = s, a_h^{k^\prime} = a\right\}}{N_h^k(s,a)} \,, \nonumber\\ \hat{P}_h^{o,k}(s,a) =& \frac{\sum^{k - 1}_{k^\prime = 1} \mathbb{I} \left\{s_h^{k^\prime} = s, a_h^{k^\prime} = a
    , s_{h+1}^{k^\prime} = s^\prime\right\}}{N_h^k(s,a)} \,,
\end{align}
where $N_h^k(s,a) = \max \left\{ \sum^{k-1}_{k^\prime = 1}  \mathbb{I}\left\{s_h^{k^\prime} = s, a_h^{k^\prime} = a\right\},1\right\}$.

\paragraph{Challenges in Optimistic Robust Policy Evaluation}
In each episode, the algorithm estimates $Q$-values with an optimistic variant of the bellman equation. 
Specifically, to encourage exploration in the robust MDP, we add a bonus term $b_h^k(s,a)$, which compensates for the lack of knowledge
of the actual reward and transition model as well as the uncertainly set, with order $b_h^k(s,a) = O\left(1 / \sqrt{N_h^k(s,a)} \right)$.
\begin{align*}
    \hat{Q}^{k}_h (s,a) = \min\left\{\hat{r}(s,a) + \sigma_{\hat{\gP}_h}(\hat{V}_{h+1}^\pi)(s) + b_h^k(s,a), H\right\} \,.
\end{align*}
Intuitively, the bonus term $b_h^k$ desires to characterize the optimism required for efficient exploration for both the estimation errors of $P$ and the robustness of $P$. 
It is hard to control the two quantities in their primal form because of the coupling between them. We propose the following procedure to address the problem.

Note that the key difference between our algorithm and standard policy optimization is that $\sigma_{\hat{\gP}_h}(\hat{V}_{h+1}^\pi)(s)$ requires solving an inner minimization (\ref{eq:sigma}). 
Through relaxing the constraints with Lagrangian multiplier and Fenchel conjugates, under $(s,a)$-rectangular set, the inner minimization problem can be reduced to a one-dimensional unconstrained convex optimization problem on $\mathbb{R}$ (Lemma \ref{lem:sa_con}). 
\begin{align}\label{eq:inner_sa}
    \sup_{\eta} \eta - \frac{ (\eta - \min_s \limits \hVpn(s))_{+}}{2}\rho - \sum_{s^\prime}\hat{P}_h^o(s^\prime \mid s,a) \left( \eta - \hVpn(s^\prime)\right)_{+} \,.
\end{align}
The optimum of Equation (\ref{eq:inner_sa}) is then computed efficiently with bisection or sub-gradient methods. We note that while the dual form has been similarly used before under the presence of a generative model or with an offline dataset \citep{badrinath2021robust,panaganti2022sample,yang2021towards}, it remains unclear whether it is effective for the online setting.

Similarly, in the case of $s$-rectangular set, the inner minimization problem is equivalent to a $A$-dimensional convex optimization problem.
\begin{align}\label{eq:inner_s}
    \sup_{\eta} \ \sum_{a^\prime} \eta_{a^\prime} -  \sum_{s^\prime, a^\prime} \hat{P}_h^o(s^\prime \mid s,a^\prime) \left(\eta_{a^\prime} - \mathbb{I}\{a^\prime = a\} V_{h+1}^{\pi_k}(s^\prime) \right)_{+} - \min_{s^\prime, a^\prime}\frac{A \rho  (\eta_{a^\prime} - \mathbb{I}\{a^\prime = a\} V_{h+1}^{\pi_k}(s^\prime))_{+}}{2} \,.
\end{align}
This optimum in $\mathbb{R}^A$ can be computed efficiently in $\tilde{O}(A)$ iterations by methods like gradient descent.


In addition to reducing computational complexity, the dual form (Equation (\ref{eq:inner_sa}) and Equation (\ref{eq:inner_s})) decouples the uncertainty in estimation error and in robustness, as $\rho$ and $\hat{P}_h^o$ are not in different terms. The exact form of $b_h^k$ is presented in the Equation (\ref{bonus_sa}) and (\ref{bonus_s}).

\paragraph{Policy Improvement Step}
Using the optimistic $Q$-value obtained from policy evaluation, the algorithm improves the policy with a KL regularized online mirror descent step, 
\begin{align*}
    \pi_h^{k+1} \in \arg\min_{\pi } \limits \beta \langle \nabla \hat{V}_{h}^{\pi_k}, \pi \rangle- \pi_h^k + D_{KL} (\pi || \pi_h^k) \,,
\end{align*}
where $\beta$ is the learning rate.
Equivalently, the updated policy is given by the closed-form solution
\begin{align*}
    \pi_h^{k+1}(a \mid s) = \frac{\pi_h^{k}\exp(\beta \hat{Q}^{\pi}_h (s,a))}{\sum_{a^\prime} \exp(\beta \hat{Q}^{\pi}_h (s,a^\prime))} \,.
\end{align*}
An important property of policy improvement is to use a fundamental inequality (\ref{eq:omd}) of online mirror descent presented in \citep{shani2020optimistic}. We suspect that other online algorithms with sublinear regret could also be used in policy improvement.

In the non-robust case, this improvement step is also shown to be theoretically efficient \citep{shani2020optimistic,wu2022nearly}. Many empirically successful policy optimization algorithms, such as PPO \citep{schulman2017proximal} and TRPO \cite{schulman2015trust}, also take a similar approach to KL regularization for non-robust policy improvement.

The proposed algorithm is summarized in Algorithm \ref{alg}.
\begin{algorithm}[h]
    \caption{Robust Optimistic Policy Optimization (ROPO)} 
    \label{alg}
    \begin{algorithmic}
    \STATE Input: learning rate $\beta$, bonus function $b_h^k$.
        \FOR{$k = 1, \ldots, K$}
        \STATE Collect a trajectory of samples by executing $\pi_k$.
        \STATE{ {\color{gray}\# Robust Policy Evaluation }}
        \FOR{$h = H, \ldots, 1$}
        \FOR{ $\forall (s,a) \in \gS \times \gA$}
        \STATE Solve $\sigma_{\hat{\gP}_h}(\hat{V}_{h+1}^\pi)(s,a)$ according to Equation (\ref{eq:inner_sa}) for $(s,a)$-rectangular set \\ or Equation (\ref{eq:inner_s}) for $s$-rectangular set.
        \STATE $\hat{Q}^{k}_h (s,a) = \min\left\{\hat{r}(s,a) + \sigma_{\hat{\gP}_h}(\hat{V}_{h+1}^\pi)(s,a) + b_h^k(s,a), H\right\} $.
        \ENDFOR
        \FOR{ $\forall s \in \gS$}
        \STATE $\hat{V}_h^k(s) = \left\langle \hat{Q}_h^k(s, \cdot), \pi_h^k(\cdot \mid s) \right\rangle$.
        \ENDFOR
        \ENDFOR
        \STATE{ {\color{gray} \# Policy Improvement}}
        \FOR{$\forall h, s, a \in [H] \times \gS \times \gA$}
        \STATE $\pi_h^{k+1}(a \mid s) = \frac{\pi_h^{k}\exp(-\beta \hat{Q}^{\pi}_h (s,a))}{\sum_{a^\prime} \exp(-\beta \hat{Q}^{\pi}_h (s,a^\prime))} $.
        \ENDFOR
        \STATE Update empirical estimate $\hat{r}$, $\hat{P}$ with Equation (\ref{eq:empirical}).
        \ENDFOR
    \end{algorithmic}
\end{algorithm}


\section{Theoretical results}

We are now ready to analyze the theoretical results of our algorithm under the uncertainly set.

\subsection{Results under $(s,a)$-rectangular uncertainty set} Equipped with Algorithm \ref{alg} and the bonus function described in Equation \ref{bonus_sa}. 
We obtain the regret upper bound under $(s,a)$-rectangular uncertainty set described in the following Theorem.

\begin{restatable*}[Regret under $(s,a)$-rectangular uncertainty set]{thm}{sa}
\label{thm:sa}
With learning rate $\beta = \sqrt{\frac{2 \log A}{H^2 K}}$ and bonus term $b_h^k$ as (\ref{bonus_sa}), with probability at least $ 1 - \delta$, the regret incurred by Algorithm \ref{alg} over $K$ episodes is bounded by 
\begin{align*}
    \text{Regret}(K) 
    = O \left( H^2  S \sqrt{AK\log \left( SAH^2 K^{3/2} ( 1 + \rho) / \delta \right)}\right) \,.
\end{align*}
\end{restatable*}

\begin{rem}
When $\rho = 0$, the problem reduces to non-robust reinforcement learning. In such case our regret upper bound is $\tilde{O}\left( H^2 S \sqrt{AK} \right)$, which is in the same order of policy optimization algorithms for the non-robust case \citet{shani2020optimistic}. 
\end{rem}
While we defer the detailed proof to the appendix \ref{appendix:thm1}, we remark on the techniques used in our proof. 

The main challenge of deriving a robust regret is to quantify the uncertainty of the transition. In the non-robust case, this uncertainty is solely incurred by limited interaction with the environment. However, in the robust case, the uncertainty is caused by both the limited interaction and the uncertainty set. With the compound causes of uncertainty we choose not to directly use concentration inequality $\sigma_{\hat{\gP}_{(s,a)}}(V) - \sigma_{\gP_{(s,a)}}(V)$ and instead resort to the dual form Equation (\ref{eq:inner_sa}). 

Notice that now the difference of $\sigma_{\hat{\gP}_{(s,a)}}(V) - \sigma_{\gP_{(s,a)}}(V)$ is only incurred by difference in the value of $\sum_{s^\prime}P_h^o(s^\prime \mid s,a) \left( \eta - \hVpn(s^\prime)\right)_{+}$. When $\eta$ is bounded, we can use Hoeffding's inequality to control it. We then investigate the range of possible optimal values of $\eta$ and use an $\epsilon$-net argument.

Our algorithm and analysis techniques can also extend to other uncertainty sets, such as KL divergence constrained uncertainly set. We include the KL divergence result in Appendix \ref{appendix:thm3}. 

\subsection{Results under $s$-rectangular uncertainty set}
Beyond the $(s,a)$-rectangular uncertainty set, we also extends to $s$-rectangular uncertainty set (Definition \ref{def:s}).
Recall that value-based methods do not extend to $s$-rectangular uncertainty set as there might not exist a deterministic optimal policy.

\begin{restatable*}[Regret under $s$-rectangular uncertainty set]{thm}{s}
\label{thm:s}
With learning rate $\beta = \sqrt{\frac{2 \log A}{H^2 K}}$ and bonus term $b_h^k$ as (\ref{bonus_s}), with probability at least $ 1 - \delta$, the regret of Algorithm \ref{alg} is bounded by 
\begin{align*}
    \text{Regret}(K) 
    = O \left( SA^2 H^2\sqrt{K\log(SA^2H^2K^{3/2}(1+\rho) / \delta)}\right)  \,.
\end{align*}
\end{restatable*}
\begin{rem}
When $\rho = 0$, the problem reduces to non-robust reinforcement learning. In such case our regret upper bound is $\tilde{O}\left( SA^2 H^2 \sqrt{K} \right)$. Our result is the first theoretical result for learning a robust policy under $s$-rectangular uncertainty set, as previous results only learn the robust value function \citep{yang2021towards}. 
\end{rem}

The analysis and techniques used for Theorem \ref{thm:s} hold great similarity to those ones used for Theorem \ref{thm:sa}. The main difference is on bounding $\hPs (\hVpn)(s,a)  - \sPs(\hVpn)(s,a)$. As the robustness of $\hPs (\hVpn)(s,a)$ is no longer independent for different actions, we can not reduce the problem of $\hPs (\hVpn)(s,a)$ into a scalar optimization problem. Instead, through analyzing the Lagrangian form, we obtain the $A$-dimensional convex optimization problem (\ref{eq:inner_s}), which is solvable in $O(A)$. Different from the $(s,a)$-rectangular case, our Lagrangian form has $A$ times more variables, which complicates the solution regions of the optimum. Through proof by contradiction argument, we find the optimal ranges of each dual variable separately. With the optimum of $\eta$, we can apply concentration inequalities uniformly over the range of dual variables. The theorem follows the same arguments of Theorem \ref{thm:sa} thereafter.

\section{Empirical results} 
To validate our theoretical findings, we conduct a preliminary empirical analysis of our purposed robust policy optimization algorithm.  

\begin{wrapfigure}[15]{l}{5cm}
\centering
\includegraphics[width=0.3\textwidth]{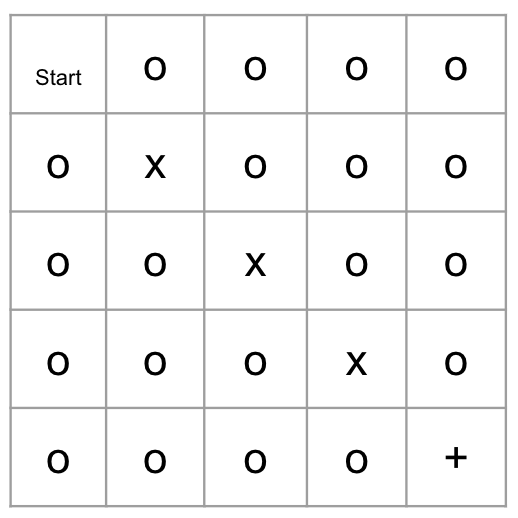}
\caption{Example of the Gridworld environment.}\label{fig:grid}
\end{wrapfigure}
\paragraph{Environment}
We conduct the experiments with the Gridworld environment, which is an early example of reinforcement learning from \cite{sutton2018reinforcement}. The environment is two-dimensional and is in a cell-like environment. Specifically, the environment is a $5 \times 5$ grid, where the agent starts from the upper left cell. The cells consist of three types, road (labeled with $o$), wall (labeled with $x$), or reward state (labeled with $+$). The agent can safely walk through the road cell but not the wall cell. Once the agent steps on the reward cell, it will receive a reward of 1, and it will receive no rewards otherwise. The goal of the agents is to collect as many rewards as possible within the allowed time. The agent has four types of actions at each step, up, down, left, and right. After taking the action, the agent has a success probability of $p$ to move according to the desired direction, and with the remaining probability of moving to other directions. 

\paragraph{Robust MDP}
To simulate the robust MDP, we create a nominal transition dynamic with success probability $p = 0.9$. The learning agent will interact with this nominal transition during training time and interact with a perturbed transition dynamic during evaluation. The transitions are perturbed along the direction is agent is directing with a constraint of $\rho$ under $(s,a)$-rectangular set. Figure \ref{fig:grid} shows an example of our environment, where the perturbation caused some of the optimal policies under nominal transition to be sub-optimal under robust transitions. 
We denote the perturbed transition as robust transitions in our results.
\paragraph{Algorithm configuration}
We implement our proposed robust policy optimization algorithm along with the non-robust variant of it \cite{shani2020optimistic}. The inner minimization of our Algorithm \ref{alg} is computed through its dual formulation for efficiency. Our algorithm is implemented with the rLberry framework \citep{rlberry}.

\paragraph{Results}
We present results with $\rho = 0.1, 0.2, 0.3$ here in Figure \ref{fig:exp}. We present the averaged cumulative rewards during evaluation. Regardless of the level of uncertainty, we observe that the robust variant of the policy optimization algorithm is more robust to dynamic changes as it is able to obtain a higher level of rewards than its non-robust variant.


\begin{figure}[h]
     \centering
     \begin{subfigure}[b]{0.3\textwidth}
         \centering
         \includegraphics[width=\textwidth]{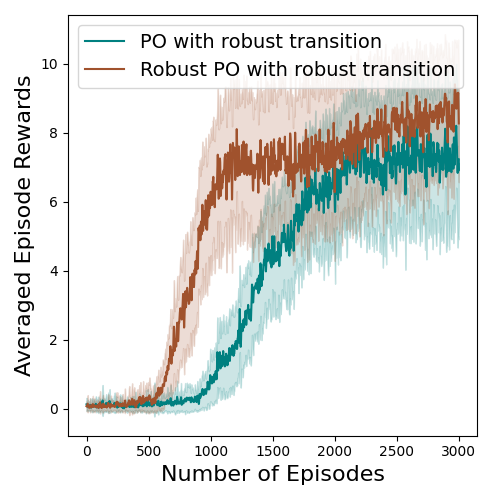}
         \caption{$\rho = 0.1$}
         \label{fig:rho1}
     \end{subfigure}
     \hfill
     \begin{subfigure}[b]{0.3\textwidth}
         \centering
         \includegraphics[width=\textwidth]{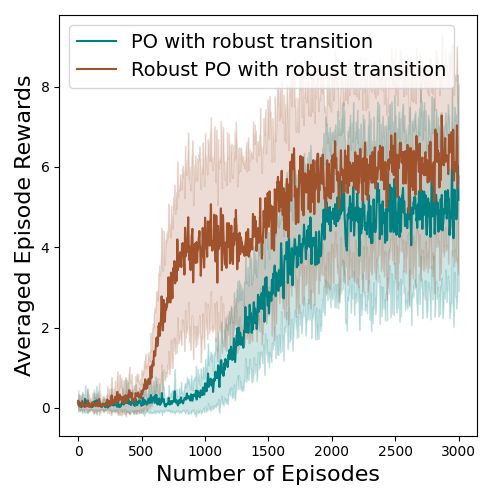}
          \caption{$\rho = 0.2$}
         \label{fig:rho2}
     \end{subfigure}
     \hfill
     \begin{subfigure}[b]{0.3\textwidth}
         \centering
         \includegraphics[width=\textwidth]{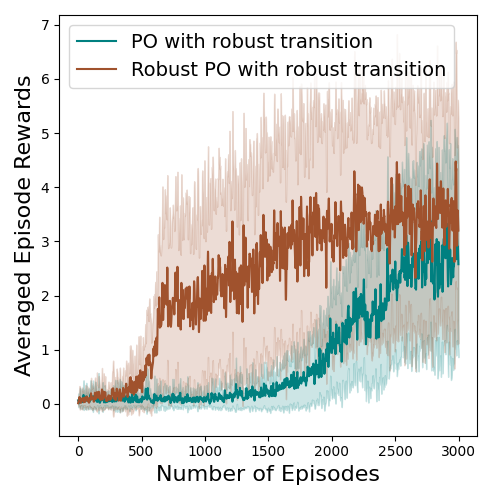}
          \caption{$\rho = 0.3$}
         \label{fig:rho3}
     \end{subfigure}
        \caption{Cumulative rewards obtained by robust and non-robust policy optimization on robust transition with different level of uncertainty $\rho = 0.1, 0.2, 0.3$ under $\ell_1$ distance.}
        \label{fig:exp}
\end{figure}

\section{Conclusion and future directions}

In this paper, we studied the problem of regret minimization in robust MDP with a rectangular uncertainty set. We proposed a robust variant of optimistic policy optimization, which achieves sublinear regret in all uncertainty sets considered. Our algorithm delicately balances the exploration-exploitation trade-off through a carefully designed bonus term, which quantifies not only the uncertainty due to the limited observations but also the uncertainty of robust MDPs. Our results are the first regret upper bounds in robust MDPs as well as the first non-asymptotic results in robust MDPs without access to a generative model. 

For future works,
while our analysis achieves the same bound as the policy optimization algorithm in ~\cite{shani2020optimistic} when the robustness level $\rho=0$, we suspect some technical details could be improved. For example, we required $P_h^o$ to be positive for any $s,a$ so that we could do a change of variable to form an efficiently solvable Fenchel dual. However, the actual positive value gets canceled out later and does not show up in the bound, suggesting that the strictly positive assumption might be an artifact of analysis. 

Furthermore, our work could also be extended in several directions. One is to consider other characterization of uncertainty sets, such as the Wasserstein distance metric. Another direction is to extend robust MDPs to a wider family of MDPs, such as the MDP with infinitely many states and with function approximation.


\bibliography{ref}

\begin{thebibliography}{39}
\providecommand{\natexlab}[1]{#1}
\providecommand{\url}[1]{\texttt{#1}}
\expandafter\ifx\csname urlstyle\endcsname\relax
  \providecommand{\doi}[1]{doi: #1}\else
  \providecommand{\doi}{doi: \begingroup \urlstyle{rm}\Url}\fi

\bibitem[Agarwal et~al.(2019)Agarwal, Jiang, Kakade, and
  Sun]{agarwal2019reinforcement}
Alekh Agarwal, Nan Jiang, Sham~M Kakade, and Wen Sun.
\newblock Reinforcement learning: Theory and algorithms.
\newblock \emph{CS Dept., UW Seattle, Seattle, WA, USA, Tech. Rep}, pages
  10--4, 2019.

\bibitem[Badrinath and Kalathil(2021)]{badrinath2021robust}
Kishan~Panaganti Badrinath and Dileep Kalathil.
\newblock Robust reinforcement learning using least squares policy iteration
  with provable performance guarantees.
\newblock In \emph{International Conference on Machine Learning}, 2021.

\bibitem[Bartlett(2013)]{bartlett2013theoretical}
Peter Bartlett.
\newblock Theoretical statistics. lecture 12, 2013.

\bibitem[Cai et~al.(2020)Cai, Yang, Jin, and Wang]{cai2020provably}
Qi~Cai, Zhuoran Yang, Chi Jin, and Zhaoran Wang.
\newblock Provably efficient exploration in policy optimization.
\newblock In \emph{International Conference on Machine Learning}, 2020.

\bibitem[Chen et~al.(2021{\natexlab{a}})Chen, Luo, and Wei]{chen2021minimax}
Liyu Chen, Haipeng Luo, and Chen-Yu Wei.
\newblock Minimax regret for stochastic shortest path with adversarial costs
  and known transition.
\newblock In \emph{Conference on Learning Theory}, 2021{\natexlab{a}}.

\bibitem[Chen et~al.(2021{\natexlab{b}})Chen, Du, and
  Jamieson]{chen2021improved}
Yifang Chen, Simon Du, and Kevin Jamieson.
\newblock Improved corruption robust algorithms for episodic reinforcement
  learning.
\newblock In \emph{International Conference on Machine Learning},
  2021{\natexlab{b}}.

\bibitem[Cobbe et~al.(2019)Cobbe, Klimov, Hesse, Kim, and
  Schulman]{cobbe2019quantifying}
Karl Cobbe, Oleg Klimov, Chris Hesse, Taehoon Kim, and John Schulman.
\newblock Quantifying generalization in reinforcement learning.
\newblock In \emph{International Conference on Machine Learning}, 2019.

\bibitem[Dann et~al.(2017)Dann, Lattimore, and Brunskill]{dann2017unifying}
Christoph Dann, Tor Lattimore, and Emma Brunskill.
\newblock Unifying pac and regret: Uniform pac bounds for episodic
  reinforcement learning.
\newblock \emph{Advances in Neural Information Processing Systems}, 2017.

\bibitem[Domingues et~al.(2021)Domingues, Flet-Berliac, Leurent, M{\'e}nard,
  Shang, and Valko]{rlberry}
Omar~Darwiche Domingues, Yannis Flet-Berliac, Edouard Leurent, Pierre
  M{\'e}nard, Xuedong Shang, and Michal Valko.
\newblock {rlberry - A Reinforcement Learning Library for Research and
  Education}, 10 2021.
\newblock URL \url{https://github.com/rlberry-py/rlberry}.

\bibitem[Even-Dar et~al.(2004)Even-Dar, Kakade, and Mansour]{even2004experts}
Eyal Even-Dar, Sham~M Kakade, and Yishay Mansour.
\newblock Experts in a markov decision process.
\newblock \emph{Advances in neural information processing systems}, 2004.

\bibitem[Farebrother et~al.(2018)Farebrother, Machado, and
  Bowling]{farebrother2018generalization}
Jesse Farebrother, Marlos~C Machado, and Michael Bowling.
\newblock Generalization and regularization in dqn.
\newblock \emph{arXiv preprint arXiv:1810.00123}, 2018.

\bibitem[Iyengar(2005)]{iyengar2005robust}
Garud~N Iyengar.
\newblock Robust dynamic programming.
\newblock \emph{Mathematics of Operations Research}, 30\penalty0 (2):\penalty0
  257--280, 2005.

\bibitem[Jin et~al.(2020)Jin, Jin, Luo, Sra, and Yu]{pmlr-v119-jin20c}
Chi Jin, Tiancheng Jin, Haipeng Luo, Suvrit Sra, and Tiancheng Yu.
\newblock Learning adversarial {M}arkov decision processes with bandit feedback
  and unknown transition.
\newblock In \emph{International Conference on Machine Learning}, 2020.

\bibitem[Jin and Luo(2020)]{jin2020simultaneously}
Tiancheng Jin and Haipeng Luo.
\newblock Simultaneously learning stochastic and adversarial episodic mdps with
  known transition.
\newblock \emph{Advances in neural information processing systems}, 2020.

\bibitem[Kallus et~al.(2022)Kallus, Mao, Wang, and Zhou]{kallus2022doubly}
Nathan Kallus, Xiaojie Mao, Kaiwen Wang, and Zhengyuan Zhou.
\newblock Doubly robust distributionally robust off-policy evaluation and
  learning.
\newblock \emph{International Conference on Machine Learning}, 2022.

\bibitem[Lykouris et~al.(2021)Lykouris, Simchowitz, Slivkins, and
  Sun]{lykouris2021corruption}
Thodoris Lykouris, Max Simchowitz, Alex Slivkins, and Wen Sun.
\newblock Corruption-robust exploration in episodic reinforcement learning.
\newblock In \emph{Conference on Learning Theory}, 2021.

\bibitem[Ma et~al.(2022)Ma, Liang, Xia, Zhang, Blanchet, Liu, Zhao, and
  Zhou]{ma22distribution}
Xiaoteng Ma, Zhipeng Liang, Li~Xia, Jiheng Zhang, Jose Blanchet, Mingwen Liu,
  Qianchuan Zhao, and Zhengyuan Zhou.
\newblock Distributionally robust offline reinforcement learning with linear
  function approximation.
\newblock \emph{arXiv preprint arXiv:2209.06620}, 2022.

\bibitem[Mannor et~al.(2012)Mannor, Mebel, and Xu]{mannor2012lightning}
Shie Mannor, Ofir Mebel, and Huan Xu.
\newblock Lightning does not strike twice: robust mdps with coupled
  uncertainty.
\newblock In \emph{Proceedings of the 29th International Coference on
  International Conference on Machine Learning}, 2012.

\bibitem[Neu et~al.(2010)Neu, Antos, Gy{\"o}rgy, and
  Szepesv{\'a}ri]{neu2010online}
Gergely Neu, Andras Antos, Andr{\'a}s Gy{\"o}rgy, and Csaba Szepesv{\'a}ri.
\newblock Online markov decision processes under bandit feedback.
\newblock \emph{Advances in Neural Information Processing Systems}, 2010.

\bibitem[Nilim and El~Ghaoui(2005)]{nilim2005robust}
Arnab Nilim and Laurent El~Ghaoui.
\newblock Robust control of markov decision processes with uncertain transition
  matrices.
\newblock \emph{Operations Research}, 53\penalty0 (5):\penalty0 780--798, 2005.

\bibitem[Packer et~al.(2018)Packer, Gao, Kos, Kr{\"a}henb{\"u}hl, Koltun, and
  Song]{packer2018assessing}
Charles Packer, Katelyn Gao, Jernej Kos, Philipp Kr{\"a}henb{\"u}hl, Vladlen
  Koltun, and Dawn Song.
\newblock Assessing generalization in deep reinforcement learning.
\newblock \emph{arXiv preprint arXiv:1810.12282}, 2018.

\bibitem[Panaganti and Kalathil(2022)]{panaganti2022sample}
Kishan Panaganti and Dileep Kalathil.
\newblock Sample complexity of robust reinforcement learning with a generative
  model.
\newblock In \emph{International Conference on Artificial Intelligence and
  Statistics}, 2022.

\bibitem[Qi and Liao(2020)]{qi2020robust}
Zhengling Qi and Peng Liao.
\newblock Robust batch policy learning in markov decision processes.
\newblock \emph{arXiv preprint arXiv:2011.04185}, 2020.

\bibitem[Raileanu and Fergus(2021)]{raileanu2021decoupling}
Roberta Raileanu and Rob Fergus.
\newblock Decoupling value and policy for generalization in reinforcement
  learning.
\newblock In \emph{International Conference on Machine Learning}, 2021.

\bibitem[Rosenberg and Mansour(2019)]{rosenberg2019online}
Aviv Rosenberg and Yishay Mansour.
\newblock Online convex optimization in adversarial markov decision processes.
\newblock In \emph{International Conference on Machine Learning}, 2019.

\bibitem[Satia and Lave~Jr(1973)]{satia1973markovian}
Jay~K Satia and Roy~E Lave~Jr.
\newblock Markovian decision processes with uncertain transition probabilities.
\newblock \emph{Operations Research}, 21\penalty0 (3):\penalty0 728--740, 1973.

\bibitem[Schulman et~al.(2015)Schulman, Levine, Abbeel, Jordan, and
  Moritz]{schulman2015trust}
John Schulman, Sergey Levine, Pieter Abbeel, Michael Jordan, and Philipp
  Moritz.
\newblock Trust region policy optimization.
\newblock In \emph{International conference on machine learning}, 2015.

\bibitem[Schulman et~al.(2017)Schulman, Wolski, Dhariwal, Radford, and
  Klimov]{schulman2017proximal}
John Schulman, Filip Wolski, Prafulla Dhariwal, Alec Radford, and Oleg Klimov.
\newblock Proximal policy optimization algorithms.
\newblock \emph{arXiv preprint arXiv:1707.06347}, 2017.

\bibitem[Shani et~al.(2020)Shani, Efroni, Rosenberg, and
  Mannor]{shani2020optimistic}
Lior Shani, Yonathan Efroni, Aviv Rosenberg, and Shie Mannor.
\newblock Optimistic policy optimization with bandit feedback.
\newblock In \emph{International Conference on Machine Learning}, 2020.

\bibitem[Song et~al.(2019)Song, Jiang, Tu, Du, and
  Neyshabur]{song2019observational}
Xingyou Song, Yiding Jiang, Stephen Tu, Yilun Du, and Behnam Neyshabur.
\newblock Observational overfitting in reinforcement learning.
\newblock In \emph{International Conference on Learning Representations}, 2019.

\bibitem[Sutton and Barto(2018)]{sutton2018reinforcement}
Richard~S Sutton and Andrew~G Barto.
\newblock \emph{Reinforcement learning: An introduction}.
\newblock MIT press, 2018.

\bibitem[Wang and Zou(2021)]{wang2021online}
Yue Wang and Shaofeng Zou.
\newblock Online robust reinforcement learning with model uncertainty.
\newblock \emph{Advances in Neural Information Processing Systems}, 2021.

\bibitem[Wang and Zou(2022)]{wang2022policy}
Yue Wang and Shaofeng Zou.
\newblock Policy gradient method for robust reinforcement learning.
\newblock \emph{International Conference on Machine Learning}, 2022.

\bibitem[White~III and Eldeib(1994)]{white1994markov}
Chelsea~C White~III and Hany~K Eldeib.
\newblock Markov decision processes with imprecise transition probabilities.
\newblock \emph{Operations Research}, 42\penalty0 (4):\penalty0 739--749, 1994.

\bibitem[Wiesemann et~al.(2013)Wiesemann, Kuhn, and
  Rustem]{wiesemann2013robust}
Wolfram Wiesemann, Daniel Kuhn, and Ber{\c{c}} Rustem.
\newblock Robust markov decision processes.
\newblock \emph{Mathematics of Operations Research}, 38\penalty0 (1):\penalty0
  153--183, 2013.

\bibitem[Wu et~al.(2022)Wu, Yang, Zhong, Wang, Du, and Jiao]{wu2022nearly}
Tianhao Wu, Yunchang Yang, Han Zhong, Liwei Wang, Simon Du, and Jiantao Jiao.
\newblock Nearly optimal policy optimization with stable at any time guarantee.
\newblock In \emph{International Conference on Machine Learning}, 2022.

\bibitem[Yang et~al.(2021)Yang, Zhang, and Zhang]{yang2021towards}
Wenhao Yang, Liangyu Zhang, and Zhihua Zhang.
\newblock Towards theoretical understandings of robust markov decision
  processes: Sample complexity and asymptotics.
\newblock \emph{arXiv preprint arXiv:2105.03863}, 2021.

\bibitem[Zhang et~al.(2022)Zhang, Chen, Zhu, and Sun]{zhang2022corruption}
Xuezhou Zhang, Yiding Chen, Xiaojin Zhu, and Wen Sun.
\newblock Corruption-robust offline reinforcement learning.
\newblock In \emph{International Conference on Artificial Intelligence and
  Statistics}, 2022.

\bibitem[Zhou et~al.(2021)Zhou, Zhou, Bai, Qiu, Blanchet, and
  Glynn]{zhou2021finite}
Zhengqing Zhou, Zhengyuan Zhou, Qinxun Bai, Linhai Qiu, Jose Blanchet, and
  Peter Glynn.
\newblock Finite-sample regret bound for distributionally robust offline
  tabular reinforcement learning.
\newblock In \emph{International Conference on Artificial Intelligence and
  Statistics}, 2021.

\end{thebibliography}

\clearpage

\appendix
\section{Proofs of Theorem 1}\label{appendix:thm1}

\subsection{Good events}
We first define the following good events, in which case we estimate the reward function and the nominal transition functions fairly accurately. 
\begin{align*}
    \gG_{k}^{r} = \ & \left\{\forall s, a, h:\left|r_{h}(s, a)-\hat{r}_{h}^{k}(s, a)\right| \leq \sqrt{\frac{2 \ln (2 SAH^2K / \delta^\prime)}{N_h^k(s,a)}}\right\} \,, \\
    \gG_k^{p} = \ & \left\{\forall s, a, h:\sPsa (\hVpn)(s)  - \hPsa(\hVpn)(s)
    \leq C_h^k(s,a) \right\}\,,
\end{align*}
where $C_h^k(s,a) = H \sqrt{\frac{4 S \log(3SAH^2K^{3/2}(4 + \rho)/ \delta^\prime)}{N_h^k(s,a)}} + \frac{1}{\sqrt{K}}$.

When the two good events happens at the same time, we say the algorithm in inside the good event $\gG = \left( \bigcap^K_{k=1} \gG_{k}^{r}\right) \bigcap \left( \bigcap^K_{k=1} \gG_{k}^{p}\right)$. The following lemma shows that $\gG$ happens with high probability by setting $\delta^\prime$ properly. 

\begin{lem}[Good event]
Let $\delta = 2 \delta^{\prime}$,  then the good event happens with high probability, i.e. $\mathbb{P}\left[ \gG\right] \geq 1 - \delta$.
\end{lem}
\begin{proof}
    By Hoeffding's inequality and an union bound on all $s,a$, all possible values of $N_k(s,a)$ and $k$, we have $\mathbb{P}\left[ \bigcap^K_{k=1} \gG_{k}^{r}\right] \geq 1 - \delta^\prime$. By Lemma \ref{lem:sa_con}, we have $\mathbb{P}\left[ \bigcap^K_{k=1} \gG_{k}^{p}\right] \geq 1 - \delta^\prime$ Then set $\delta = 2\delta^\prime$ and we have the desired result. 
\end{proof}

\subsection{Design of the bonus function}

In the case of $(s,a)$-rectangular uncertainty set, we use the following bonus function $b_h^k(s,a)$ to encourage exploration. 
\begin{align}\label{bonus_sa}
    b_h^k(s,a) = \sqrt{\frac{2 \log(3SAH^2K/ \delta)}{N_h^k(s,a)}} + H \sqrt{\frac{4 S \log(3SAH^2K^{3/2}(4 + \rho)/ \delta)}{N_h^k(s,a)}} + \frac{1}{\sqrt{K}} \,. 
\end{align}

\subsection{Regret Analysis}

Armed with the defined good event, we are now ready to present the anlysis of Theorem \ref{thm:sa}, which establishes the regret of the Algorithm under $(s,a)$-uncertainty set.

\sa

\begin{proof}
    We start with decomposing the regret as follows,
    \begin{align*}
    \text{Regret}(K) 
    = \ & \sum^K_{k=1} V_1^{\ast}(s) - V_1^{\pi_k}(s)\\
    = \ & \sum^K_{k=1} \left(V_1^{\ast} (s)- \hat{V}_1^{\pi_k}(s) \right) + \left(\hat{V}_1^{\pi_k} (s) - V_1^{\pi_k}(s)\right) \,.
    \end{align*}
    By Lemma \ref{lem:pseudo_regret_sa} and Lemma \ref{lem:sa_con}, with probability at least $1 - \delta$, we have 
    \begin{align*}
        \text{Regret}(K) 
        = \ & O \left(H^2\sqrt{K \log A} \right) +O \left(H^2 S\sqrt{AK \log \left( SAH^2 K^{3/2} ( 1 + \rho) / \delta \right)} \right) \\
        = \ & O \left( H^2  S \sqrt{AK\log \left( SAH^2 K^{3/2} ( 1 + \rho) / \delta \right)}\right) \,.
    \end{align*}
\end{proof}

\begin{lem}\label{lem:pseudo_regret_sa}
With probability at least $ 1 - \delta$, we have
\begin{align*}
     \sum^K_{k=1}V_1^{\ast}(s) - \hat{V}_1^{\pi_k}(s) = O \left(H^2\sqrt{K \log A} \right)\,.
\end{align*}
\end{lem}

\begin{proof}
    For any $h \in [1, H]$, we have
\begin{align*}
    & V_h^{\ast}(s) - \hat{V}_h^{\pi_k}(s) \\
    = \ & \langle Q_h^{\ast}(s, \cdot) , \pi_\ast(\cdot \mid s) \rangle - \langle \hat{Q}_h^{\pi_k}(s, \cdot) , \pi_k(\cdot \mid s) \rangle \\
    = \ & \langle Q_h^{\ast}(s, \cdot) -  \hat{Q}_h^{\pi_k}(s, \cdot), \pi_\ast(\cdot \mid s) \rangle + \langle \hat{Q}_h^{\pi_k}(s, \cdot) , \pi_\ast(\cdot \mid s) - \pi_k(\cdot \mid s)  \rangle \\
    = \ & \mathbb{E}_{\pi_\ast} \left[ (r_h(s, a) - \hat{r}_{h}^{k}(s,a)) + (\sPsa (V_{h+1}^\ast)(s) -  \hPsa(\hVpn)(s)) - b_h^k(s,a)\right] \\
    & \ + \langle \hat{Q}_h^{\pi_k}(s, \cdot) , \pi_\ast(\cdot \mid s) - \pi_k(\cdot \mid s)  \rangle \\
    = \ & \mathbb{E}_{\pi_\ast} \left[ (r_h(s, a) - \hat{r}_{h}^{k}(s,a)) +  (\sPsa(\hVpn)(s) -  \hPsa(\hVpn)(s)) - b_h^k(s,a)\right] \\
    & \ + \mathbb{E}_{\pi_\ast} \left[  \sPsa (V_{h+1}^\ast)(s) - \sPsa(\hVpn)(s)\right] + \langle \hat{Q}_h^{\pi_k}(s, \cdot) , \pi_\ast(\cdot \mid s) - \pi_k(\cdot \mid s)  \rangle \,,
\end{align*}
where the third equality is by the update rule of our algorithm and the robust bellman equation.

By the design of our bonus function, conditioned on the good event, we have
\begin{align*}
     (r_h(s, a) - \hat{r}_{h}^{k}(s,a)) + (\sPsa (V_{h+1}^\ast)(s) -  \hPsa(\hVpn)(s)) - b_h^k(s,a) \leq 0 \,.
\end{align*}

Let $q_h(\cdot \mid s,a) = \mathop{\arg\min}_{P_h \in \gP_h} \limits P_h(\cdot \mid s,a) \hVpn $, then we have
\begin{align*}
    & \sPsa (V_{h+1}^\ast)(s) - \sPsa(\hVpn)(s)  \\
    = \ & \min_{P_h \in \gP_h} \limits P_h(\cdot \mid s,a) V_{h+1}^\ast - \min_{P_h \in \gP_h} \limits P_h(\cdot \mid s,a) \hVpn  \\
    = \ & \min_{P_h \in \gP_h} \limits P_h(\cdot \mid s,a) V_{h+1}^\ast - q_h(\cdot \mid s,a) \hVpn  \\
    \le \ & q_h(\cdot \mid s,a) (V_{h+1}^\ast -  \hVpn)  \\
     \le \ & \max_{P_h\in \gP_h} \limits P_h(\cdot \mid s,a) (V_{h+1}^\ast - \hVpn)\,.
\end{align*}

Let $p_h(\cdot \mid s,a) = \mathop{\arg\max}_{P_h \in \gP_h} \limits P_h(\cdot \mid s,a) (V_{h+1}^\ast)(s,a) $, Then we have the following relation hold conditioned on the good event:
\begin{align*}
    & V_h^{\ast}(s) - \hat{V}_h^{\pi_k}(s) \\
    \le \ & \mathbb{E}_{\pi_\ast} \left[  \sup_{P_h\in \gP_h} \limits P_h(\cdot \mid s,a) (V_{h+1}^\ast - \hVpn)\right] + \langle \hat{Q}_h^{\pi_k}(s, \cdot) , \pi_\ast(\cdot \mid s) - \pi_k(\cdot \mid s)  \rangle \\
    = \ & \mathbb{E}_{\pi_\ast, p_h} \left[ V_{h+1}^\ast(s) - \hVpn(s) \right] + \langle \hat{Q}_h^{\pi_k}(s, \cdot) , \pi_\ast(\cdot \mid s) - \pi_k(\cdot \mid s)  \rangle \,.
\end{align*}

Then, by applying above relation recursively and with the fact that for any policy $\pi$ and state $s$, $V_{H+1}^{\ast}(s) = \hat{V}_{H+1}^{\pi_k}(s) = 0$, we have 
\begin{align*}
    V_1^{\ast}(s) - \hat{V}_1^{\pi_k}(s)
    \le \sum^H_{h=1} \mathbb{E}_{\pi_\ast, \{p_{t}\}^{h-1}_{t=1}} \left[ \langle \hat{Q}_h^{\pi_k}(s, \cdot) , \pi_\ast(\cdot \mid s) - \pi_k(\cdot \mid s)  \rangle \right] \,.
\end{align*}

Summing over $k$, we get

\begin{align*}
    \sum^K_{k=1}V_1^{\ast}(s) - \hat{V}_1^{\pi_k}(s)
    & \le \sum^K_{k=1}\sum^H_{h=1} \mathbb{E}_{\pi_\ast, \{p_{t}\}^{h-1}_{t=1}} \left[ \langle \hat{Q}_h^{\pi_k}(s, \cdot) , \pi_\ast(\cdot \mid s) - \pi_k(\cdot \mid s)  \rangle \right] \\ 
    &= \sum^H_{h=1} \mathbb{E}_{\pi_\ast, \{p_{t}\}^{h-1}_{t=1}} \left[ \sum^K_{k=1}\langle \hat{Q}_h^{\pi_k}(s, \cdot) , \pi_\ast(\cdot \mid s) - \pi_k(\cdot \mid s)  \rangle \right] \,.
\end{align*}

By standard results for online mirror descent (Lemma \ref{lem:omd}), we have
\begin{align*}
    \sum^K_{k=1} \langle \hat{Q}_h^{\pi_k}(s, \cdot) , \pi_\ast(\cdot \mid s) - \pi_k(\cdot \mid s)  \rangle \leq \frac{\log (A)}{\beta} + \frac{\beta}{2} \sum^K_{k=1} \sum_{a \in \gA} \pi_h^\ast (a \mid s) (\hat{Q}_h^{\pi_k}(s, a) )^2 \,.
\end{align*}
By the update rule of Algorithm \ref{alg}, we have $0 \le \hat{Q}_h^{\pi_k}(s, a) \le H$, for all $h, k$. Then take $\beta = \sqrt{\frac{2 \log A}{H^2 K}}$, 

\begin{align*}
    \sum^K_{k=1} \langle \hat{Q}_h^{\pi_k}(s, \cdot) , \pi_\ast(\cdot \mid s) - \pi_k(\cdot \mid s)  \rangle \leq 
     \sqrt{2 H^2 K\log A}\,.
\end{align*}

Finally, we have
\begin{align*}
     \sum^K_{k=1}V_1^{\ast}(s) - \hat{V}_1^{\pi_k}(s) \le H \sqrt{2 H^2 K\log A} = O \left(H^2\sqrt{K \log A} \right)\,.
\end{align*}
\end{proof}

\begin{lem}\label{lem:estimation_sa}
With probability at least $ 1 - \delta$, we have
\begin{align*}
    \sum^K_{k=1}(\hat{V}_1^{\pi_k} - V_1^{\pi_k})(s)  = O \left(H^2 S\sqrt{AK \log \left( SAH^2 K^{3/2} ( 1 + \rho) / \delta \right)} \right)\,.
\end{align*}
\end{lem}

\begin{proof}
    By the algorithm's update rule and the robust bellman equation, we have 
\begin{align*}
    (\hVp - \Vp )(s)
    = \ & \langle \hQp (s, \cdot) - \Qp (s, \cdot) , \pi_k(\cdot \mid s) \rangle \\
    = \ & \left\langle \hat{r}_h^k(s, \cdot) - r_h^k(s, \cdot)  + ( \sigma_{\hat{\gP}_{(s,\cdot)}}(\hVpn)(s,\cdot) - \sigma_{\gP_{(s,\cdot)}}(\Vpn)(s,\cdot) ) + b_h^k(s,\cdot),  \pi_k(\cdot \mid s) \right\rangle  \\
    = \ & \mathbb{E}_{\pi_k}\left[ \hat{r}_h^k(s, a) - r_h^k(s, a)  + ( \hPsa(\hVpn)(s) - \sPsa (\Vpn)(s) ) + b_h^k(s,a) \right]\,.
\end{align*}
By adding and subtracting a term $\sPsa (\hVpn)(s, a) $, we have 
\begin{align*}
    & \hPsa(\hVpn)(s) - \sPsa (\Vpn)(s) \\
    = \ & \hPsa(\hVpn)(s)  - \sPsa (\hVpn)(s) + \sPsa (\hVpn)(s)  - \sPsa (\Vpn)(s)\\
    \leq \ &   \hPsa (\hVpn)(s) - \sPsa (\hVpn)(s) + \max_{P_h \in \gP_h}P_h(\cdot \mid s,a)( \hVpn - \Vpn) \,.
\end{align*}

Let $p_h(\cdot \mid s,a) = \mathop{\arg\max}_{P_h \in \gP_h} \limits P_h(\cdot \mid s,a)(\hVpn - \Vpn)$, we have
\begin{align*}
    &(\hVp - \Vp )(s) \\
    \leq \ & \mathbb{E}_{\pi_k}\left[ \hat{r}_h^k(s, a) - r_h^k(s, a)+ \hPsa (\hVpn)(s) - \sPsa (\hVpn)(s) + p_h(\cdot \mid s,a)(\hVpn - \Vpn)  + b_h^k(s,a) \right] \\
    = \ & \mathbb{E}_{\pi_k, p_h}\left[ \hat{r}_h^k(s, a) - r_h^k(s, a)+ \hPsa (\hVpn)(s) - \sPsa (\hVpn)(s) + \hVpn(s) - \Vpn(s)  + b_h^k(s,a) \right] 
\end{align*}

By applying the above relation recursively and with the fact that for any policy $\pi$ and state $s$, $V_{H+1}^{\pi_k}(s) =  \hat{V}_{H+1}^{\pi_k}(s) = 0$, we have 
\begin{align*}
    & (\hat{V}_1^{\pi_k} - V_1^{\pi_k})(s) 
    \leq \ &\sum^H_{h=1} \mathbb{E}_{\pi_k, \{p_t\}^h_{t=1}} \left[ \hat{r}_h^k(s, a) - r_h^k(s, a) + \hPsa(\hVpn)(s) -  \sPsa (\hVpn)(s) + b_h^k(s,a)\right]\,.
\end{align*}

Conditioned on the good even and by the design of our bonus function, we have 
\begin{align*}
    \hat{r}_h^k(s, a) - r_h^k(s, a) + \hPsa(\hVpn)(s) -  \sPsa (\hVpn)(s) \le b_h^k(s,a) \,.
\end{align*}
Then, with probability at least $1 - \delta$, we have
\begin{align*}
     \sum^K_{k=1}(\hat{V}_1^{\pi_k} - V_1^{\pi_k})(s) 
    \leq \ &\sum^K_{k=1}\sum^H_{h=1} \mathbb{E}_{\pi_k, \{p_t\}^h_{t=1}} \left[ 2 b_h^k(s,a)\right]\\
    \leq \ & H\sqrt{K}+ O \left(  H \sqrt{S \log(SAH^2K^{3/2}(4 + \rho) / \delta)} \right)\sum^K_{k=1}\sum^H_{h=1} \mathbb{E}_{\pi_k, \{p_t\}^h_{t=1}} \left[\sqrt{\frac{1}{N_h^k(s,a)}} \right]\,.
\end{align*}

By Lemma \ref{lem:visit}, we have the bound of the visitation counts:
\begin{align*}
    \sum^K_{k=1}\sum^H_{h=1} \sqrt{\frac{1}{N_h^k(s,a)}} \leq 2 H \sqrt{SAK} \,.
\end{align*}
Combining everything, with probability at least $1 - \delta$
\begin{align*}
    \sum^K_{k=1}(\hat{V}_1^{\pi_k} - V_1^{\pi_k})(s)  = O \left(H^2 S\sqrt{AK \log \left( SAH^2 K^{3/2} ( 1 + \rho) / \delta \right)} \right) \,.
\end{align*}

\end{proof}

\begin{lem}\label{lem:sa_con}
For any $h,k,s,a$, the following inequality holds with probability at least $1-\delta^\prime$,
\begin{align*}
    \sPsa (\hVpn)(s)  - \hPsa(\hVpn)(s)
    \leq \ &   H\sqrt{\frac{4 S \log(3SAH^3K^{3/2}(4+\rho)/ \delta^\prime)}{N_h^k(s,a)}} + \frac{1}{H\sqrt{K}}\,.
\end{align*}
\end{lem}
\begin{proof}
    By the definition of $ \sPsa (\hVpn)(s) 
    = \  \min_{P_h \in \gP_h} \limits \sum_{s^\prime} P_h(s^\prime \mid s,a) \hVpn(s^\prime) $, we have the following optimization problem:
\begin{equation*}
\begin{split}
&\min_{P_h} \,\, \sum_{s^\prime} P_h(s^\prime \mid s,a) \hVpn(s^\prime) \\
&\text{s.t.}\quad  \left\{\begin{array}{lc}
\sum_{s^\prime} | P_h(s^\prime \mid s,a) - P_h^o(s^\prime \mid s,a)| \leq \rho \,, \\
\sum_{s^\prime} P_h(s^\prime \mid s,a) = 1 \,, \\
P_h^o(\cdot \mid s,a) > 0, P_h(\cdot \mid s,a) \ge 0  \,. 
\end{array}\right.
\end{split}
\end{equation*}

Define $\Tilde{P}_h(s^\prime \mid s, a) = \frac{ P_h(s^\prime \mid s,a) }{ P_h^o(s^\prime \mid s,a)}$, we can rewrite the above optimization problem as
\begin{equation*}
\begin{split}
&\min_{\Tilde{P}_h } \sum_{s^\prime} \Tilde{P}_h(s^\prime \mid s,a) P_h^o(s^\prime \mid s,a) \hVpn(s^\prime) \\
&\text{s.t.}\quad  \left\{\begin{array}{lc}
\sum_{s^\prime} | \Tilde{P}_h(s^\prime \mid s,a)  - 1| P_h^o(s^\prime \mid s,a) \leq \rho \,, \\
\sum_{s^\prime} \Tilde{P}_h(s^\prime \mid s,a) P_h^o(s^\prime \mid s,a) = 1 \,, \\
\Tilde{P}_h(s^\prime \mid s,a) \ge 0 \quad \forall s^\prime \in \gS\,. 
\end{array}\right.
\end{split}
\end{equation*}

Using the Lagrangian multiplier method, we have the following Lagrangian $L(\Tilde{P}_h, \eta, \lambda)$ with Lagrangian multiplier $\eta\in \mathbb{R}, \lambda\geq 0$,
\begin{align*}
    L(\Tilde{P}_h, \eta, \lambda) (s,a) 
    = \ & \sum_{s^\prime} \Tilde{P}_h(s^\prime \mid s,a) P_h^o(s^\prime \mid s,a) \hVpn(s^\prime) + \lambda \left( \sum_{s^\prime} | \Tilde{P}_h(s^\prime \mid s,a)  - 1| P_h^o(s^\prime \mid s,a)  - \rho \right) \\
    & \ - \eta \left(\sum_{s^\prime} \Tilde{P}_h(s^\prime \mid s,a) P_h^o(s^\prime \mid s,a) - 1 \right) \\
    = \ &  \eta - \lambda \rho - \lambda \sum_{s^\prime}P_h^o(s^\prime \mid s,a) \left(\frac{\eta}{\lambda}\Tilde{P}_h(s^\prime \mid s,a) - | \Tilde{P}_h(s^\prime \mid s,a)  - 1| - \frac{\Tilde{P}_h(s^\prime \mid s,a) \hVpn(s^\prime)}{\lambda} \right)  \\
    = \ &  \eta - \lambda \rho - \lambda \sum_{s^\prime}P_h^o(s^\prime \mid s,a) \left(\frac{\eta - \hVpn(s^\prime) }{\lambda}\Tilde{P}_h(s^\prime \mid s,a) - | \Tilde{P}_h(s^\prime \mid s,a)  - 1| \right)  \,.
\end{align*}

We define $f(x) = |x - 1|$ and the convex conjugate is $f^\ast(y) = \max_{x} \limits  \langle x, y\rangle - f(x) $. 
Let $x$ be $\Tilde{P}_h$ and by using $f^\ast$, we can optimize over $\Tilde{P}_h$ and rewrite the Lagrangian as 
\begin{align*}
    L(\eta, \lambda) (s,a) = \min_{\Tilde{P}_h} \limits L(\Tilde{P}_h, \eta, \lambda) (s,a) 
    =  \eta - \lambda \rho - \lambda \sum_{s^\prime}P_h^o(s^\prime \mid s,a) f^\ast \left(\frac{\eta - \hVpn(s^\prime)}{\lambda} \right) \,.
\end{align*}

Notice that conditioned on $x \ge 0$,  $f(x) = |x-1|$'s convex conjugate has the following closed form:
\begin{equation*}
f^\ast(y)=\max_{x} \limits  \langle x, y\rangle - f(x) =\left\{
\begin{aligned}
-1& \quad  \text{$y \le -1$} \,,\\
y&  \quad \text{$y \in [-1,1]$} \,,\\
+\infty&  \quad \text{$y > 1$} \,.
\end{aligned}
\right.
\end{equation*}

Let $\Tilde{\eta} = \eta + \lambda$, then using the closed form of $f^\ast(y)$, the equality $\max\left\{a,b\right\} = (a - b)_{+} + b$ and condition on $\frac{\eta - \hVpn(s^\prime)}{\lambda} \le 1$,  we can rewrite the optimization problem as
\begin{align*}
    L(\Tilde{\eta}, \lambda) (s,a)
   = \ & \eta - \lambda \rho - \lambda \sum_{s^\prime}P_h^o(s^\prime \mid s,a) f^\ast \left(\frac{\eta - \hVpn(s^\prime)}{\lambda} \right) \\
   = \ &\Tilde{\eta} - \lambda -  \lambda \rho - \lambda \sum_{s^\prime}P_h^o(s^\prime \mid s,a) \max\left\{ \frac{\eta - \hVpn(s^\prime)}{\lambda}, -1 \right\} \\
   = \ & \Tilde{\eta} - \lambda -  \lambda \rho - \lambda \sum_{s^\prime}P_h^o(s^\prime \mid s,a) \left(\left( \frac{\eta - \hVpn(s^\prime)}{\lambda} - (-1)\right)_{+} + (-1)\right)\\
   = \ & \Tilde{\eta} - \lambda -  \lambda \rho -  \sum_{s^\prime}P_h^o(s^\prime \mid s,a) (\Tilde{\eta} - \hVpn(s^\prime))_{+} + \lambda\\
   = \ & \Tilde{\eta} - \lambda \rho -  \sum_{s^\prime}P_h^o(s^\prime \mid s,a) (\Tilde{\eta} - \hVpn(s^\prime))_{+} \,.
\end{align*}
with the constraint of $\lambda$ being 
$$\lambda \ge 0,   \quad \Tilde{\eta} - \min_s \limits \hVpn(s) \leq 2 \lambda .$$

Note that $L(\Tilde{\eta}, \lambda) (s,a)$ is inversely proportional to $\lambda$, it achieves the maximum when $\lambda = 
\frac{(\Tilde{\eta} - \min_s \limits \hVpn(s))_{+}}{2}$. By directly optimizing it over $\lambda$, we can reduce the problem to 
\begin{align*}
    L(\Tilde{\eta}) (s,a) = \Tilde{\eta} - \frac{ (\Tilde{\eta} - \min_s \limits \hVpn(s))_{+}}{2}\rho - \sum_{s^\prime}P_h^o(s^\prime \mid s,a) ( \Tilde{\eta} - \hVpn(s^\prime))_{+} \,.
\end{align*}

Define the function $g$ as 
\[
g(\Tilde{\eta}, P_h^o) = - L(\Tilde{\eta})(s,a) = \sum_{s^\prime}P_h^o(s^\prime \mid s,a) \left( \Tilde{\eta} - \hVpn(s^\prime)\right)_{+}  - \Tilde{\eta} + \frac{ (\Tilde{\eta} - \min_s\limits  \hVpn(s))_{+}}{2} \rho \,.
\]

Then we investigate the optimum of $g$.
First notice that $g(0) = 0$, when $\Tilde{\eta} \leq 0$, $g(\Tilde{\eta}, P_h^o) = - \Tilde{\eta} \geq 0$.

On the other hand, when $\Tilde{\eta} \ge H$, 
\begin{align*}
    g(\Tilde{\eta}, P_h^o) 
    = \ &\sum_{s^\prime}P_h^o(s^\prime \mid s,a) ( \Tilde{\eta} - \hVpn(s^\prime))  - \Tilde{\eta} + \frac{ (\Tilde{\eta} - \min_s\limits  \hVpn(s))}{2} \rho \\
    = \ & -\sum_{s^\prime}P_h^o(s^\prime \mid s,a) \hVpn(s^\prime)  + \frac{ (\Tilde{\eta} - \min_s\limits  \hVpn(s))}{2} \rho \,.
\end{align*}

Note that now $g$ is directly proportional to $\Tilde{\eta}$, therefore $g$ achieves the minimum within the range of $\Tilde{\eta} \in [0, H]$. We remark that the same form is also used for analyzing robust policy evaluation (Lemma B.1 \citep{yang2021towards}).

With this, we can rewrite 
\begin{align*}
    \hPsa(\hVpn)(s) - \sPsa (\hVpn)(s)  
    = \ & -\min_{\eta_1 \in [0, H] } g(\eta_1, \hat{P}_h^{o,k}) + \min_{\eta_2 \in [0, H] }g\left(\eta_2, P_h^o\right)\\
    \leq \ & \max_{\eta \in [0, H ]}| g\left(\eta, \hat{P}_h^{o,k}\right)  - g\left(\eta, P_h^o\right) |\,.
\end{align*}

To upper bound $ \hPsa(\hVpn)(s) - \sPsa (\hVpn)(s) $, we first upper bound $| g\left(\eta, \hat{P}_h^{o,k}\right)  - g\left(\eta, P_h^o\right) |$.
\begin{align*}
    | g\left(\eta, \hat{P}_h^{o,k}\right)  - g\left(\eta, P_h^o\right) |
    = \ & \left|  \sum_{s^\prime}\hat{P}_h^{o,k}(s^\prime \mid s,a) \left( \eta - \hVpn(s^\prime)\right)_{+}  - \sum_{s^\prime}P_h^o(s^\prime \mid s,a) \left( \eta - \hVpn(s^\prime)\right)_{+} \right|  \\
    \leq \ & \left\|\hat{P}_h^{o,k}(\cdot \mid s,a) -  P_h^o(\cdot \mid s,a)\right\|_1 \ \max_{s \in \gS} \limits | \eta - \hVpn(s)|_\infty \\ 
    \leq \ &  H \left\|\hat{P}_h^{o,k}(\cdot \mid s,a) -  P_h^o(\cdot \mid s,a)\right\|_1 \,,
\end{align*}
where the first inequality is by Cauchy-Schwarz inequality, the second inequality follows from $\eta \in [0, H ]$.

By Hoeffding's inequality and an union bound over all $s,a$, the following inequality holds with probability at least $1 - \delta^\prime$:
\begin{align*}
    \left\|\hat{P}_h^{o,k}(\cdot \mid s,a) -  P_h^o(\cdot \mid s,a)\right\|_1 \leq \sqrt{\frac{4 S \log(3SAH^2K/ \delta^\prime)}{N_h^k(s,a)}} \,.
\end{align*}

To upper bound the error with maximum over $\eta$, we first create an $\epsilon$-net $N_\epsilon(\eta)$ with $g$ over $\eta\in [0, H ]$ such that
\begin{align*}
    \max_{\eta\in [0, H ]} |  g\left(\eta, \hat{P}_h^{o,k}\right)  - g\left(\eta, P_h^o\right)| \leq \max_{\eta\in N_\epsilon(\eta)} |  g\left(\eta, \hat{P}_h^{o,k}\right)  - g\left(\eta, P_h^o\right)| + 2 \epsilon \,.
\end{align*}

By taking an union bound over $N_\epsilon(\eta)$, we have 
\begin{align*}
    \max_{\eta\in [0, H ]} |  g\left(\eta, \hat{P}_h^{o,k}\right)  - g\left(\eta, P_h^o\right)| \leq H \sqrt{\frac{4 S \log(3SAH^2K | N_\epsilon(\eta)| / \delta^\prime)}{N_h^k(s,a)}} + 2 \epsilon \,,
\end{align*}
where $| N_\epsilon(\eta)| $ is the size of the $\epsilon$-net.

It now remains to bound the size of $| N_\epsilon(\eta)| $, which can be obtained easily if $g$ is Lischitz. 
Notice that
\begin{align*}
     |g(\Tilde{\eta_1}, P_h^o) -  g(\Tilde{\eta_2}, P_h^o) | 
     \le \ &  \sum_{s^\prime}P_h^o(s^\prime \mid s,a) | \Tilde{\eta_1} - \Tilde{\eta_2}|  + | \Tilde{\eta_1} - \Tilde{\eta_2}| + \frac{ |\Tilde{\eta_1} - \Tilde{\eta_2}|}{2} \rho \\
     = \ & \frac{4 + \rho}{2} |\Tilde{\eta_1} - \Tilde{\eta_2}| \,, 
\end{align*}
where the first inequality is by the absolute inequality and 
$|(a)_{+}-(b)_{+}| \le |a-b|$.

Then $g$ is a $\frac{4 + \rho}{2}$-Lipschitz function over $\eta\in [0, H ]$, thus combined with Lemma \ref{lem:eps_cover}, we have $| N_\epsilon(\eta)| = O\left(\frac{4 + \rho}{2 \epsilon}\right)$. Hence, we have the following inequality happens with at least $1-\delta^\prime$ probability:
\begin{align*}
    \max_{\eta\in [0, H ]} |  g\left(\eta, \hat{P}_h^{o,k}\right)  - g\left(\eta, P_h^o\right)| \leq  H \sqrt{\frac{4 S \log(3SAH^2K (4 + \rho)/2\epsilon \delta^\prime)}{N_h^k(s,a)}} + 2 \epsilon\,.
\end{align*}
Take $\epsilon = \frac{1}{2\sqrt{K}}$, we have the following inequality happens with at least $1-\delta^\prime$ probability:
\begin{align*}
    \sPsa (\hVpn)(s)  - \hPsa(\hVpn)(s)
    \leq \ & \max_{\eta\in [0, H ]} |  g\left(\eta, \hat{P}_h^{o,k}\right)  - g\left(\eta, P_h^o\right)| \\
    \leq \ &  H \sqrt{\frac{4 S \log(3SAH^2K^{3/2}(4 + \rho)/ \delta^\prime)}{N_h^k(s,a)}} + \frac{1}{\sqrt{K}}\,.
\end{align*}
\end{proof}

\newpage
\section{Proof of Theorem 2}\label{appendix:thm2}

\subsection{Good events}
We first define the following good events, in which case we estimate the reward function and the nominal transition functions fairly accurately. 
\begin{align*}
    \gG_{k}^{r} = \ & \left\{\forall s, a, h:\left|r_{h}(s, a)-\hat{r}_{h}^{k}(s, a)\right| \leq \sqrt{\frac{2 \ln (2 SAH^2K / \delta^\prime)}{N_h^k(s,a)}}\right\} \,, \\
    \gG_k^{p} = \ & \left\{\forall s, a, h: \sPs (\hVpn)(s)  - \hPs(\hVpn)(s)
    \leq C_h^k(s,a)\right\}\,,
\end{align*}
where \[C_h^k(s,a) = AH\sqrt{\frac{4 S A\log(3SA^2H^3K^{3/2}(4+\rho) / \delta^\prime)}{N_h^k(s,a)}} +  \frac{1}{H\sqrt{K}} \,. \]
When the two good events happens at the same time, we say the algorithm in inside the good event $\gG = \left( \bigcap^K_{k=1} \gG_{k}^{r}\right) \bigcap \left( \bigcap^K_{k=1} \gG_{k}^{p}\right)$. The following lemma shows that $\gG$ happens with high probability.

\begin{lem}[Good event]
Let $\delta = 2 \delta^{\prime}$,  then the good event happens with high probability, i.e. $\mathbb{P}\left[ \gG\right] \geq 1 - \delta$.
\end{lem}

\begin{proof}
    By Hoeffding's inequality and an union bound on all $s,a$, all possible values of $N_k(s,a)$ and $k$, we have $\mathbb{P}\left[ \bigcap^K_{k=1} \gG_{k}^{r}\right] \geq 1 - \delta^\prime$. By Lemma \ref{lem:con_s}, we have $\mathbb{P}\left[ \bigcap^K_{k=1} \gG_{k}^{p}\right] \geq 1 - \delta^\prime$ Then set $\delta = 2\delta^\prime$ and we have the desired result. 
\end{proof}

\subsection{Design of the bonus function}

In the case of $s$-rectangular uncertainty set, we use the following bonus function $b_h^k(s,a)$ to encourage exploration. 
\begin{align}\label{bonus_s}
    b_h^k(s,a) =  AH\sqrt{\frac{4 S A\log(3SA^2H^2K^{3/2}(4+\rho) / \delta)}{N_h^k(s,a)}} +  \frac{1}{\sqrt{K}}   + \sqrt{\frac{2 \log(3SAH^2K/ \delta^\prime)}{N_h^k(s,a)}}  \,.
\end{align}
\subsection{Regret analysis}
\s

\begin{proof}
    Similar to the case of $(s,a)$-rectangular set, we start with decomposing the regret as follows,
    \begin{align*}
    \text{Regret}(K) 
    = \ & \sum^K_{k=1} V_1^{\ast}(s) - V_1^{\pi_k}(s)\\
    = \ & \sum^K_{k=1} \left(V_1^{\ast} (s)- \hat{V}_1^{\pi_k}(s) \right) + \left(\hat{V}_1^{\pi_k} (s) - V_1^{\pi_k}(s)\right) \,.
    \end{align*}
    By Lemma \ref{lem:pseudo_regret_sa} and Lemma \ref{lem:estimation_s}, with probability at least $1 - \delta$, we have 
    \begin{align*}
        \text{Regret}(K) 
        = \ & O \left(H^2\sqrt{K \log A} \right) + O \left( SA^2 H^2\sqrt{K\log(SA^2H^2K^{3/2}(1+\rho) / \delta)}\right)  \\
        = \ & O \left( SA^2 H^2\sqrt{K\log(SA^2H^2K^{3/2}(1+\rho) / \delta)}\right)  \,.
    \end{align*}
\end{proof}

\begin{lem}\label{lem:estimation_s}
    With Algorithm \ref{alg}, we have
\begin{align*}
    \sum^K_{k=1}(\hat{V}_1^{\pi_k} - V_1^{\pi_k})(s) = O \left( SA^2 H^2\sqrt{K\log(SA^2H^2K^{3/2}(1+\rho) / \delta)}\right)  \,.
\end{align*}
\end{lem}

\begin{proof}
    Similar to the case with $(s,a)$-rectangular uncertainty set, for any $k$, we can decompose $(\hat{V}_1^{\pi_k} - \hat{V}_1^{\pi_k})(s)$ as, 
\begin{align*}
    & (\hat{V}_1^{\pi_k} - \hat{V}_1^{\pi_k})(s) \\
    \leq \ & \sum^H_{h=1} \mathbb{E}_{\pi_k, \{p_t\}^h_{t=1}} \left[ (r_h^k(s, a) - \hat{r}_h^k(s, a)) + \left( \hPs \left(\hVpn\right)(s)  - \sPs\left(\hVpn\right)(s)\right) + b_h^k(s,a)\right]\,.
\end{align*}

Thus by the design of our bonus function and with probability at least $ 1 - \delta$, we have
\begin{align*}
    & \sum^K_{k=1}(\hat{V}_1^{\pi_k} - V_1^{\pi_k})(s) \\
    \leq \ & 2\sum^K_{k=1}\sum^H_{h=1} \mathbb{E}_{\pi_k, \{p_t\}^h_{t=1}} \left[b_h^k(s,a) \right] \\
    = \ & H\sqrt{K} + O\left(  HA\sqrt{SA\log(SA^2H^2K^{3/2}(1+\rho) / \delta)}\right)\sum^K_{k=1}\sum^H_{h=1} \mathbb{E}_{\pi_k, \{p_t\}^h_{t=1}} \left[ \sqrt{\frac{1}{N_h^k(s,a)}}\right]  \,.
\end{align*}

By Lemma \ref{lem:visit}, 
we have the bound of visitation counts:
\begin{align*}
    \sum^K_{k=1}\sum^H_{h=1} \sqrt{\frac{1}{N_h^k(s,a)}} \leq 2 H \sqrt{SAK} \,.
\end{align*}
Combining everything, conditioned on the good event we have
\begin{align*}
    \sum^K_{k=1}(\hat{V}_1^{\pi_k} - V_1^{\pi_k})(s) = O \left( SA^2 H^2\sqrt{K\log(SA^2H^2K^{3/2}(1+\rho) / \delta)}\right) \,.
\end{align*}

\end{proof}

\begin{lem}\label{lem:con_s}
    For any $h,k,s,a$, the following inequality holds with probability at least $1-\delta$,
    \begin{align*}
    \hPs (\hVpn)(s)  - \sPs (\hVpn)(s)
    \leq \ & AH\sqrt{\frac{4 S A\log(3SA^2H^2K^{3/2}(4+\rho) / \delta)}{N_h^k(s,a)}} +  \frac{1}{\sqrt{K}} \,.
\end{align*}
\end{lem}
\begin{proof}
  By the definition of $ \sPs (\hVpn)(s) 
    =  \inf_{P_h \in \gP_h} \limits \sum_{s^\prime} P_h(s^\prime \mid s,a) \hVpn(s^\prime) $, we consider the following optimization problem:
\begin{equation*}
\begin{split}
&\min_{P_h} \,\, \sum_{s^\prime} P_h(s^\prime \mid s,a) \hVpn(s^\prime) \\
&\text{s.t.}\quad  \left\{\begin{array}{lc}
\sum_{s^\prime, a^\prime} | P_h(s^\prime \mid s,a^\prime) - P_h^o(s^\prime \mid s,a^\prime)| \leq A\rho \,, \\
\sum_{s^\prime} P_h(s^\prime \mid s,a^\prime) = 1 \,,  \forall a^\prime \in \gA \,, \\
P_h^o(\cdot \mid s,a^\prime) > 0, P_h(\cdot \mid s,a^\prime) \ge 0  \,,  \forall a^\prime \in \gA \,. \\
\end{array}\right.
\end{split}
\end{equation*}

Let $\Tilde{P}_h(s^\prime \mid s, a) = \frac{ P_h(s^\prime \mid s,a) }{ P_h^o(s^\prime \mid s,a)}$, we can rewrite the above optimization problem as
\begin{equation*}
\begin{split}
&\min_{\Tilde{P}_h } \sum_{s^\prime} \Tilde{P}_h(s^\prime \mid s,a) P_h^o(s^\prime \mid s,a) \hVpn(s^\prime) \\
&\text{s.t.}\quad  \left\{\begin{array}{lc}
\sum_{s^\prime, a^\prime} | (\Tilde{P}_h(s^\prime \mid s,a^\prime) - 1| P_h^o(s^\prime \mid s,a^\prime) \leq A\rho \,, \\
\sum_{s^\prime} \Tilde{P}_h(s^\prime \mid s,a^\prime)P_h^o(s^\prime \mid s,a^\prime) = 1 \,, \quad  \forall a^\prime \in \gA\\
\Tilde{P}_h(\cdot \mid s,a^\prime) \geq 0\,, \quad  \forall a^\prime \in \gA\,.
\end{array}\right.
\end{split}
\end{equation*}

Use the Lagrangian multiplier method and $f(x) = |x - 1|$, we have the Lagrangian $L(\Tilde{P}_h, \eta, \lambda)$ with multiplier $\eta = \{\eta_a\}_{a\in\gA}, \eta_a \in \mathbb{R}$, $\lambda \geq 0$,
\begin{align*}
     & L\left(\Tilde{P}_h, \eta, \lambda\right) (s,a) \\
    = \ & \sum_{s^\prime} \Tilde{P}_h(s^\prime \mid s,a) P_h^o(s^\prime \mid s,a) \hVpn(s^\prime) + \lambda \left( \sum_{s^\prime, a^\prime} \left| (\Tilde{P}_h(s^\prime \mid s,a^\prime) - 1\right| P_h^o(s^\prime \mid s,a^\prime) - A\rho \right)\\ 
    & - \sum_{a^\prime} \eta_{a^\prime} \left( \sum_{s^\prime} \Tilde{P}_h(s^\prime \mid s,a^\prime)P_h^o(s^\prime \mid s,a^\prime) - 1 \right) \\
     = \ &  - \lambda A \rho + \sum_{a^\prime} \eta_{a^\prime} + 
    \lambda \sum_{s^\prime, a^\prime} P_h^o(s^\prime \mid s,a^\prime) \left( f\left(\Tilde{P}_h(s^\prime \mid s,a^\prime) \right) - \Tilde{P}_h(s^\prime \mid s,a^\prime) \left(\frac{\eta_{a^\prime} - \mathbb{I}\{a^\prime = a\} V_{h+1}^{\pi_k}(s^\prime)}{\lambda} \right)\right) \,.
\end{align*}

The convex conjugate of $f$ is $f^\ast(y) = \max_{x} \limits \langle x, y\rangle - f(x) $. 
Using $f^\ast$, we can thus optimize over $\Tilde{P}_h$ and rewrite the Lagrangian over as
\begin{align*}
    L(\eta, \lambda) (s,a) = \ &  \min_{\Tilde{P}_h} \limits L\left(\Tilde{P}_h, \eta, \lambda\right) (s,a) \\
    = \ & - \lambda A \rho +  \sum_{a^\prime} \eta_{a^\prime} - \lambda \sum_{s^\prime, a^\prime} P_h^o(s^\prime \mid s,a^\prime) f^\ast \left( \frac{\eta_{a^\prime} - \mathbb{I}\{a^\prime = a\} V_{h+1}^{\pi_k}(s^\prime)}{\lambda}\right) \,.
\end{align*}

Conditioned on $x \ge 0$, $f(x) = |x-1|$, notice that the conjugate $f^\ast(y)$ has the following closed form,
\begin{equation*}
f^\ast(y)=\max_{x} \limits  \langle x, y\rangle - f(x) =\left\{
\begin{aligned}
-1& \quad  \text{$y \le -1$} \,,\\
y&  \quad \text{$y \in [-1,1]$} \,,\\
+\infty&  \quad \text{$y > 1$} \,.
\end{aligned}
\right.
\end{equation*}

Let $\Tilde{\eta}_a = \eta_a + \lambda$, using the closed form of $f^\ast(y)$, the equality $\max\left\{a,b\right\} = (a - b)_{+} + b$ and conditioned on $\frac{\eta_{a^\prime} - \mathbb{I}\{a^\prime = a\} V_{h+1}^{\pi_k}(s^\prime)}{\lambda} \le 1$,  we can rewrite the optimization problem as
\begin{align*}
    L(\Tilde{\eta}, \lambda) (s,a) &= - \lambda A \rho +  \sum_{a^\prime} \eta_{a^\prime} - \lambda \sum_{s^\prime, a^\prime} P_h^o(s^\prime \mid s,a^\prime) f^\ast \left( \frac{\eta_{a^\prime} - \mathbb{I}\{a^\prime = a\} V_{h+1}^{\pi_k}(s^\prime)}{\lambda}\right) \\
    &= - \lambda A \rho -\lambda A +  \sum_{a^\prime} \Tilde{\eta}_{a^\prime} - \lambda \sum_{s^\prime, a^\prime} P_h^o(s^\prime \mid s,a^\prime) \max \left\{ \frac{\eta_{a^\prime} - \mathbb{I}\{a^\prime = a\} V_{h+1}^{\pi_k}(s^\prime)}{\lambda}, -1\right\} \\
    &= - \lambda A \rho + \sum_{a^\prime} \Tilde{\eta}_{a^\prime} -  \sum_{s^\prime, a^\prime} P_h^o(s^\prime \mid s,a^\prime) \left(\Tilde{\eta}_{a^\prime} - \mathbb{I}\{a^\prime = a\} V_{h+1}^{\pi_k}(s^\prime) \right)_{+} \,.
\end{align*}
where constraint of $\lambda$ is 
\[\lambda \ge 0,   \quad \Tilde{\eta}_{a^\prime} - \mathbb{I}\{a^\prime = a\} V_{h+1}^{\pi_k}(s^\prime)\leq 2\lambda, \ \forall a^\prime,s^\prime \,.\]

Note that the above Lagrangian is inversely proportional to $\lambda$ and it achieves the maximum when $\lambda = 
\max_{s^\prime, a^\prime} \limits \frac{ (\Tilde{\eta}_{a^\prime} - \mathbb{I}\{a^\prime = a\} V_{h+1}^{\pi_k}(s^\prime))_{+}}{2} $. 
Directly optimize over $\lambda$,  we can reduce the problem to
\begin{align*}
    L(\Tilde{\eta}) (s,a) = \sum_{a^\prime} \Tilde{\eta}_{a^\prime} -  \sum_{s^\prime, a^\prime} P_h^o(s^\prime \mid s,a^\prime) \left(\Tilde{\eta}_{a^\prime} - \mathbb{I}\{a^\prime = a\} V_{h+1}^{\pi_k}(s^\prime) \right)_{+} - \max_{s^\prime, a^\prime}\frac{A \rho (\Tilde{\eta}_{a^\prime} - \mathbb{I}\{a^\prime = a\} V_{h+1}^{\pi_k}(s^\prime))_{+}}{2} \,.
\end{align*}
Define $g\left(\Tilde{\eta}, P_h^o\right) = -L(\Tilde{\eta}) (s,a) $ as 
\begin{align*}
    g(\Tilde{\eta}, P_h^o) 
    = \ &  - \sum_{a^\prime} \Tilde{\eta}_{a^\prime} +  \sum_{s^\prime, a^\prime} P_h^o(s^\prime \mid s,a^\prime) \left(\Tilde{\eta}_{a^\prime} - \mathbb{I}\{a^\prime = a\} V_{h+1}^{\pi_k}(s^\prime) \right)_{+} + \max_{s^\prime, a^\prime} \frac{A \rho (\Tilde{\eta}_{a^\prime} - \mathbb{I}\{a^\prime = a\} V_{h+1}^{\pi_k}(s^\prime))_{+}}{2} \,.
\end{align*}

Assume $g$ achieves its minimum when $\Tilde{\eta} = \left\{\Tilde{\eta}_1, \cdots , \Tilde{\eta}_A\right\}$.
Suppose $\Tilde{\eta}$ has a component $\Tilde{\eta}_a < 0$. Consider $\eta^\prime = \left\{\Tilde{\eta}_1, \cdots, 0, \cdots, \Tilde{\eta}_a\right\}$, where we change the zero element $\Tilde{\eta}_a$ to 0 and keep other components unchanged. Then we have
$$g(\Tilde{\eta}, P_h^o) - g(\eta^\prime, P_h^o) = -\Tilde{\eta}_A > 0\,, $$
which contradict with the hypothesis that $g$ achieves its minimum in $\Tilde{\eta}$.

On the other hand, suppose $\Tilde{\eta} $ has a component $\Tilde{\eta}_a > H$. Then consider $\eta^\prime = \left\{\Tilde{\eta}_1, \cdots, H, \cdots, \Tilde{\eta}_a\right\}$, where we change corresponding $\Tilde{\eta}_a$ to 0 and keep other components unchanged. Denote $f(\Tilde{\eta}) = \max_{s^\prime, a^\prime} \limits \frac{A \rho (\Tilde{\eta}_{a^\prime} - \mathbb{I}\{a^\prime = a\} V_{h+1}^{\pi_k}(s^\prime))_{+}}{2}$, and we have

\begin{align*}
    g\left(\Tilde{\eta}, P_h^o\right) - g\left(\eta^\prime, P_h^o\right) 
    =\ & -\Tilde{\eta}_A + H + \sum_{s^\prime} P_h^o(s^\prime \mid s,a) (\Tilde{\eta}_{a} - H)  + f(\Tilde{\eta}) - f(\eta^\prime)\\
    \ge \ & -\Tilde{\eta}_A + H + \sum_{s^\prime} P_h^o(s^\prime \mid s,a) (\Tilde{\eta}_{a} - H) \\
    = \ & 0 \,.
\end{align*}
Therefore, $g$ achieves its minimum with $\Tilde{\eta}$, with $0 \le \eta_{a} \le H, \forall a \in \gA$. We remark that a similar form and technique are also used for analyzing robust policy evaluation (Lemma C.1 \citep{yang2021towards}).

We can now rewrite 
\begin{align*}
     \hPs \left(\hVpn\right)(s)  - \sPs\left(\hVpn\right)(s)
     = \ & \min_{\eta_{1}\in [0, H]^{|\gA|}} g(\eta_1, \hat{P}_h^{o,k}) - \min_{\eta_{2}\in [0, H]^{|\gA|}} g(\eta_2, P_h^o) \\
     \leq \ & \max_{\eta\in [0, H]^{|\gA|}} \left|g\left(\eta, \hat{P}_h^{o,k}\right) -  g\left(\eta, P_h^o\right) \right| \,.
\end{align*}

To upper bound $\hPs \left(\hVpn\right)(s)  - \sPs\left(\hVpn\right)(s)$, we first consider the bound of $\left|g\left(\eta, \hat{P}_h^{o,k}\right) -  g\left(\eta, P_h^o\right) \right|$,
\begin{align*}
    &  \left|g\left(\eta, \hat{P}_h^{o,k}\right) -  g\left(\eta, P_h^o\right) \right| \\
    = \ & \left| \sum_{s^\prime, a^\prime} \hat{P}_h^{o,k}(s^\prime \mid s,a^\prime) \left(\eta_{a^\prime} - \mathbb{I}\{a^\prime = a\} V_{h+1}^{\pi_k}(s^\prime) \right)_{+} - \sum_{s^\prime, a^\prime} P_h^o(s^\prime \mid s,a^\prime) \left(\eta_{a^\prime} - \mathbb{I}\{a^\prime = a\} V_{h+1}^{\pi_k}(s^\prime) \right)_{+} \right|\\
    = \ & \left| \  \sum_{a^\prime}  \sum_{s^\prime} \left( \hat{P}_h^{o,k}(s^\prime \mid s,a^\prime) - P_h^o(s^\prime \mid s,a^\prime) \right)\left(\eta_{a^\prime} - \mathbb{I}\{a^\prime = a\} V_{h+1}^{\pi_k}(s^\prime) \right)_{+} \right|  \\
    \leq \ & \sum_{a^\prime} \left\| \hat{P}_h^{o,k}(\cdot \mid s,a^\prime) - P_h^o(\cdot \mid s,a^\prime) \right\|_1 \max_{s \in \gS}\left| \eta_{a^\prime} - \mathbb{I}\{a^\prime = a\} V_{h+1}^{\pi_k}(s)  \right|\\
    \le \ & H \sum_{a^\prime} \left\| \hat{P}_h^{o,k}(\cdot \mid s,a^\prime) - P_h^o(\cdot \mid s,a^\prime) \right\|_1\,,
\end{align*}
where the first inequality is by Cauchy-Schwarz inequality, the second inequality follows from $\eta_a \in [0, H ] , \ \forall a \in \gA$.

By Hoeffding's inequality and an union bound over all $s,a^\prime$, $N_h^k(s,a)$, the following inequality holds with probability at least $1 - \delta$,
\begin{align*}
    \left\| \hat{P}_h^{o,k}(\cdot \mid s,a^\prime) - P_h^o(\cdot \mid s,a^\prime) \right\|_1  
    \leq \ &  \sqrt{\frac{4S \log(SAH^2 K / \delta)}{N_h^k(s,a)}} \,.
\end{align*}

To upper bound $\max_{\eta\in [0, H]^{|\gA|}} \left|g\left(\eta, \hat{P}_h^{o,k}\right) -  g\left(\eta, P_h^o\right) \right|$, we first create an $\epsilon$-net $N_\epsilon(\eta)$ with $g$ over $\eta\in [0, H ]$ such that
\begin{align*}
    \max_{\eta\in [0, H ]} \left|  g\left(\eta, \hat{P}_h^{o,k}\right)  - g\left(\eta, P_h^o\right)\right| \leq \max_{\eta\in N_\epsilon(\eta)} \left|  g\left(\eta, \hat{P}_h^{o,k}\right)  - g\left(\eta, P_h^o\right)\right| + 2 \epsilon \,.
\end{align*}

Taking an union bound over $N_\epsilon(\eta)$, we have 
\begin{align*}
    \max_{\eta\in [0, H ]} \left|  g\left(\eta, \hat{P}_h^{o,k}\right)  - g\left(\eta, P_h^o\right)\right| \leq HA \sqrt{\frac{4 S \log(3SAH^2K | N_\epsilon(\eta)| / \delta)}{N_h^k(s,a)}} + 2 \epsilon \,,
\end{align*}
where $| N_\epsilon(\eta)|$ is the size of the $\epsilon$-net.

It now remains to find the size of the $\epsilon$-net, which can be easily obtained if $g$ is Lipschitz. Notice that 
\begin{align*}
     &|g(\Tilde{\eta}_1, P_h^o) -  g(\Tilde{\eta}_2, P_h^o) | \\
     \le \  &  \sum_{s^\prime, a^\prime}P_h^o(s^\prime \mid s,a) | \Tilde{\eta}_{1, a^\prime} - \Tilde{\eta}_{2, a^\prime}|  + \sum_{a^\prime}| \Tilde{\eta}_{1, a^\prime} - \Tilde{\eta}_{2, a^\prime}| + \frac{ \max_{a^\prime}\limits| \Tilde{\eta}_{1, a^\prime} - \Tilde{\eta}_{2, a^\prime}|}{2} A\rho \\
     \le \ & \frac{A(4 + \rho)}{2} \|\Tilde{\eta_1} - \Tilde{\eta_2}\|_{\infty} \,, 
\end{align*}
where the first inequality is by the absolute inequality, the property of maximum function and $|(a)_{+}-(b)_{+}| \le |a-b|$, the second inequality follows from the definition of infinity norm.

Therefore $g$ is a $\frac{A(4 + \rho)}{2}$-Lipschitz function over $\eta\in [0, H ]$. Thus combining with Lemma \ref{lem:eps_cover}, we have $| N_\epsilon(\eta)| \leq \left( \frac{A(4 + \rho)}{2\epsilon}\right)^A$. 
Hence, we have the following inequality happens with at least $1-\delta^\prime$ probability:
\begin{align*}
\hPs (\hVpn)(s)  - \sPs (\hVpn)(s) 
    \leq \ & \max_{\eta_{a} \in [0, H]^{|\gA|}} \left|  g\left(\eta, \hat{P}_h^{o,k}\right)  - g\left(\eta, P_h^o\right)\right| \\
    \leq \ & AH\sqrt{\frac{4 S A\log(3SA^2H^2K ( 4 + \rho)/ 2\epsilon \delta^\prime)}{N_h^k(s,a)}} + 2 \epsilon \,.
\end{align*}

Take $\epsilon = \frac{1}{2\sqrt{K}}$, then 
\begin{align*}
    \hPs (\hVpn)(s)  - \sPs (\hVpn)(s)
    \leq \ & AH\sqrt{\frac{4 S A\log(3SA^2H^2K^{3/2}(4+\rho) / \delta^\prime)}{N_h^k(s,a)}} +  \frac{1}{\sqrt{K}} \,.
\end{align*}
\end{proof}

\newpage
\section{Extension to uncertainty set with KL divergence}\label{appendix:thm3}

In this section, we extend our algorithm and analysis to uncertainty sets with KL divergence as a distance metric. We first formally define the uncertainty set considered, which is similar to the one in Definition \ref{def:sa}. 
\begin{defn}[$(s,a)$-rectangular uncertainty set \citet{iyengar2005robust,wiesemann2013robust}]\label{def:kl}
For all time step $h$ and with a given state-action pair $(s,a)$, the $(s,a)$-rectangular uncertainty set $\gP_h(s,a)$ is defined as 
\[
\gP_h(s,a) = \left\{\text{D}_{KL}\left(P_h(\cdot \mid s,a),  P_h^o(\cdot \mid s,a)\right) \leq \rho \,,P_h(\cdot \mid s,a) \in \Delta(\gS) \right\} \,,
\]
where $P_h^o$ is the nominal transition kernel at $h$, $P_h^o(\cdot \mid s,a) > 0, \forall (s,a) \in \gS \times \gA$, $\rho$ is the level of uncertainty and $\text{D}_{KL}\left(p(\cdot \mid s,a), q(\cdot \mid s, a)\right) = \sum_{s^\prime \in \gS} p(s^\prime \mid s,a) \log \left( \frac{ p(s^\prime \mid s,a)}{ q(s^\prime\mid s,a)}\right)$.
\end{defn}

With the above described uncertainty set, our algorithm solves $\sigma_{\hat{\gP}_h}(\hat{V}_{h+1}^\pi)(s,a)$ by solving the following sub-problem, 
\begin{align*}
    \min_{\lambda} \lambda \rho + \lambda \log\left(\sum_{s^\prime} \hat{P}_h^o(s^\prime \mid s,a)\exp\left(\frac{ -  \hat{V}_{h+1}^{\pi_k}(s^\prime) }{\lambda}\right)\right)\,.
\end{align*}
Our algorithm also uses the following bonus function in the robust policy evaluation step, 
\begin{align*}
    b_h^k(s,a) 
    = \ & C_h^k(s,a)+ \sqrt{\frac{2 \log(3SAH^2K/ \delta^\prime)}{N_h^k(s,a)}}  \,.
\end{align*}

With these modifications to algorithm \ref{alg}, the following theorem states the formal regret guarantee. 
\begin{thm}[Regret under KL divergence $(s,a)$-rectangular uncertainty set]
\label{thm:kl}
Setting the learning rate $\beta = \sqrt{\frac{2 \log A}{H^2 K}}$, then with probability at least $ 1 - \delta$, the regret incurred by Algorithm over $K$ episodes is bounded by 
\begin{align*}
    \text{Regret}(K) 
    = O \left(\frac{SH}{\rho c} \sqrt{AK\log(SAH^4K^{3/2}/ \delta)}\right) \,,
\end{align*}
where $0 < c \leq 1$ the minimal element of $P_h^o$, over all $h \in [H]$.
\end{thm}
In the following, we present the detailed analysis of Theorem \ref{thm:kl}

\subsection{Good events}
We first define the following good events, in which case we estimate the reward function and the nominal transition functions fairly accurately. 
\begin{align*}
    \gG_{k}^{r} = \ & \left\{\forall s, a, h:\left|r_{h}(s, a)-\hat{r}_{h}^{k}(s, a)\right| \leq \sqrt{\frac{2 \ln (2 SAH^2K / \delta^\prime)}{N_h^k(s,a)}}\right\} \,, \\
    \gG_k^{p} = \ & \left\{\forall s, a, h: \sPs (\hVpn)(s)  - \hPs(\hVpn)(s)
    \leq C_h^k(s,a)\right\}\,,
\end{align*}
where \[C_h^k(s,a) =\frac{2H}{\rho c}  \sqrt{\frac{4 S  \log(8SAH^4K^2/ \delta^\prime \rho)}{N_h^k(s,a)}} +  \frac{1}{\sqrt{K}} \,,\]
and $c$ is the minimal element of $P_h^o$, over all $h \in [H]$.
When the two good events happens at the same time, we say the algorithm in inside the good event $\gG = \left( \bigcap^K_{k=1} \gG_{k}^{r}\right) \bigcap \left( \bigcap^K_{k=1} \gG_{k}^{p}\right)$. The following lemma shows that $\gG$ happens with high probability.

\begin{lem}[Good event]
Let $\delta = 2 \delta^{\prime}$,  then the good event happens with high probability, i.e. $\mathbb{P}\left[ \gG\right] \geq 1 - \delta$.
\end{lem}

\begin{proof}
    By Hoeffding's inequality and an union bound on all $s,a$, all possible values of $N_k(s,a)$ and $k$, we have $\mathbb{P}\left[ \bigcap^K_{k=1} \gG_{k}^{r}\right] \geq 1 - \delta^\prime$. By Lemma \ref{lem:kl_con_sa}, we have $\mathbb{P}\left[ \bigcap^K_{k=1} \gG_{k}^{p}\right] \geq 1 - \delta^\prime$ Then set $\delta = 2\delta^\prime$ and we have the desired result. 
\end{proof}

\subsection{Regret analysis}


\begin{proof}
    Similar to the case of $(s,a)$-rectangular set, we start with decomposing the regret as follows,
    \begin{align*}
    \text{Regret}(K) 
    = \ & \sum^K_{k=1} V_1^{\ast}(s) - V_1^{\pi_k}(s)\\
    = \ & \sum^K_{k=1} \left(V_1^{\ast} (s)- \hat{V}_1^{\pi_k}(s) \right) + \left(\hat{V}_1^{\pi_k} (s) - V_1^{\pi_k}(s)\right) \,.
    \end{align*}
    By Lemma \ref{lem:pseudo_regret_sa} and Lemma \ref{lem:kl}, with probability at least $1 - \delta$, we have 
    \begin{align*}
        \text{Regret}(K) 
        = \ & O \left(H^2\sqrt{K \log A} \right)  + O \left(\frac{SH}{\rho c} \sqrt{AK\log(SAH^4K^{3/2}/ \delta)}\right)  \\
        = \ & O \left(\frac{SH}{\rho c} \sqrt{AK\log(SAH^4K^{3/2}/ \delta)}\right)\,,
    \end{align*}
    where $c$ is the minimal element of $P_h^o$, over all $h \in [H]$.
\end{proof}

\begin{lem}\label{lem:kl}
    With Algorithm \ref{alg}, we have
\begin{align*}
    \sum^K_{k=1}(\hat{V}_1^{\pi_k} - V_1^{\pi_k})(s) = O \left(\frac{1}{\rho c} S H\sqrt{AK\log(SAH^4K^{3/2}/ \delta)}\right)  \,.
\end{align*}
\end{lem}

\begin{proof}
    Similar to the case with $(s,a)$-rectangular uncertainty set, for any $k$, we can decompose $(\hat{V}_1^{\pi_k} - \hat{V}_1^{\pi_k})(s)$ as, 
\begin{align*}
    (\hat{V}_1^{\pi_k} - \hat{V}_1^{\pi_k})(s) 
    \leq \ \sum^H_{h=1} \mathbb{E}_{\pi_k, \{p_t\}^h_{t=1}} \left[ (r_h^k(s, a) - \hat{r}_h^k(s, a)) + \left( \hPs \left(\hVpn\right)(s)  - \sPs\left(\hVpn\right)(s)\right) + b_h^k(s,a)\right]\,.
\end{align*}

Thus by the design of our bonus function and with probability at least $ 1 - \delta$, we have
\begin{align*}
    & \sum^K_{k=1}(\hat{V}_1^{\pi_k} - V_1^{\pi_k})(s) \\
    \leq \ & 2\sum^K_{k=1}\sum^H_{h=1} \mathbb{E}_{\pi_k, \{p_t\}^h_{t=1}} \left[b_h^k(s,a) \right] \\
    = \ &H\sqrt{K} + O\left( \frac{1}{\rho c}\sqrt{S\log(SAH^4K^{3/2}/ \delta)}\right)\sum^K_{k=1}\sum^H_{h=1} \mathbb{E}_{\pi_k, \{p_t\}^h_{t=1}} \left[ \sqrt{\frac{1}{N_h^k(s,a)}}\right]  \,,
\end{align*}
where $c$ is a problem dependent constant.

By Lemma \ref{lem:visit}, 
we have the bound of visitation counts:
\begin{align*}
    \sum^K_{k=1}\sum^H_{h=1} \sqrt{\frac{1}{N_h^k(s,a)}} \leq 2 H \sqrt{SAK} \,.
\end{align*}
Combining everything, conditioned on the good event we have
\begin{align*}
    \sum^K_{k=1}(\hat{V}_1^{\pi_k} - V_1^{\pi_k})(s) = O \left(\frac{SH}{\rho c} \sqrt{AK\log(SAH^4K^{3/2}/ \delta)}\right) \,.
\end{align*}

\end{proof}

\begin{lem}\label{lem:kl_con_sa}
    For any $h,k,s,a$, the following inequality holds with probability at least $1-\delta$,
    \begin{align*}
   \hPs (\hVpn)(s)  - \sPs (\hVpn)(s)
    \leq \ &  \frac{2H}{\rho c}  \sqrt{\frac{4 S  \log(8SAH^4K^2/ \delta^\prime \rho)}{N_h^k(s,a)}} +  \frac{1}{\sqrt{K}} \,.
\end{align*}
where $c$ is the minimal element of $P_h^o$.
\end{lem}
\begin{proof}
  By the definition of $ \sPs \left(\hVpn\right)(s) 
    =  \inf_{P_h \in \gP_h} \limits \sum_{s^\prime} P_h(s^\prime \mid s,a) \hVpn(s^\prime) $, we consider the following optimization problem:
\begin{equation*}
\begin{split}
&\min_{P_h} \,\, \sum_{s^\prime} P_h(s^\prime \mid s,a) \hVpn(s^\prime) \\
&\text{s.t.}\quad  \left\{\begin{array}{lc}
\sum_{s^\prime} P_h(s^\prime \mid s,a)\log\left(\frac{P_h(s^\prime \mid s,a)}{P_h^o(s^\prime \mid s,a)}\right)  \leq \rho \,, \\
\sum_{s^\prime} P_h(s^\prime \mid s,a) = 1  \,, \\
P_h^o(\cdot \mid s,a) > 0, P_h(\cdot \mid s,a) \ge 0  \,. \\
\end{array}\right.
\end{split}
\end{equation*}

Let $\Tilde{P}_h(s^\prime \mid s, a) = \frac{ P_h(s^\prime \mid s,a) }{ P_h^o(s^\prime \mid s,a)}$, we can rewrite the above optimization problem as
\begin{equation*}
\begin{split}
&\min_{\Tilde{P}_h } \sum_{s^\prime} \Tilde{P}_h(s^\prime \mid s,a) P_h^o(s^\prime \mid s,a) \hVpn(s^\prime) \\
&\text{s.t.}\quad  \left\{\begin{array}{lc}
\sum_{s^\prime}  \Tilde{P}_h(s^\prime \mid s,a^\prime) P_h^o(s^\prime \mid s,a^\prime)\log\left(\Tilde{P}_h(s^\prime \mid s, a)\right) \leq \rho \,, \\
\sum_{s^\prime} \Tilde{P}_h(s^\prime \mid s,a^\prime)P_h^o(s^\prime \mid s,a) = 1 \,, \\
\Tilde{P}_h(\cdot \mid s,a) \geq 0  \,.
\end{array}\right.
\end{split}
\end{equation*}

Use the Lagrangian multiplier method and $f(x) = x \log x$, we have the Lagrangian $L(\Tilde{P}_h, \eta, \lambda)$ with multiplier $\eta \in \mathbb{R}$, $\lambda \geq 0$,
\begin{align*}
     & L(\Tilde{P}_h, \eta, \lambda) (s,a) \\
    = \ & \sum_{s^\prime} \Tilde{P}_h(s^\prime \mid s,a) P_h^o(s^\prime \mid s,a) \hVpn(s^\prime) + \lambda \left( \sum_{s^\prime}  \Tilde{P}_h(s^\prime \mid s,a^\prime) P_h^o(s^\prime \mid s,a^\prime)\log(\Tilde{P}_h(s^\prime \mid s, a)) - \rho \right)\\ 
    & -  \eta \left( \sum_{s^\prime} \Tilde{P}_h(s^\prime \mid s,a)P_h^o(s^\prime \mid s,a) - 1 \right) \\
     = \ &  - \lambda \rho +\eta + 
    \lambda \sum_{s^\prime} P_h^o(s^\prime \mid s,a) \left( f\left(\Tilde{P}_h(s^\prime \mid s,a^\prime) \right) - \Tilde{P}_h(s^\prime \mid s,a^\prime) \left(\frac{\eta - V_{h+1}^{\pi_k}(s^\prime)}{\lambda} \right)\right) \,.
\end{align*}

The convex conjugate of $f$ is $f^\ast(y) = \max_{x} \limits \langle x, y\rangle - f(x) $. 
Using $f^\ast$, we can thus optimize over $\Tilde{P}_h$ and rewrite the Lagrangian over as
\begin{align*}
    L(\eta, \lambda) (s,a) = \min_{\Tilde{P}_h} \limits L(\Tilde{P}_h, \eta, \lambda) (s,a)  = - \lambda \rho +   \eta - \lambda \sum_{s^\prime} P_h^o(s^\prime \mid s,a) f^\ast \left( \frac{\eta - V_{h+1}^{\pi_k}(s^\prime)}{\lambda}\right) \,.
\end{align*}

Conditioned on $x \ge 0$, $f(x) = x \log x$, notice that the conjugate $f^\ast(y)$ has the following closed form,
\begin{equation*}
f^\ast(y)=\max_{x} \limits  \langle x, y\rangle - f(x) =\exp(y-1) \,.
\end{equation*}

Using the closed form of $f^\ast(y)$, we can rewrite the optimization problem as
\begin{align*}
    L(\eta, \lambda) (s,a) &= - \lambda  \rho +  \eta - \lambda \sum_{s^\prime} P_h^o(s^\prime \mid s,a) f^\ast \left( \frac{\eta -  V_{h+1}^{\pi_k}(s^\prime)}{\lambda}\right) \\
    &= - \lambda \rho +  \eta - \lambda \sum_{s^\prime} P_h^o(s^\prime \mid s,a) \exp\left(\frac{\eta -  V_{h+1}^{\pi_k}(s^\prime) - \lambda}{\lambda}\right)\,.
\end{align*}

Taking the derivative of $\eta$,

\begin{align*}
    \frac{\partial L}{\partial \eta} &= 1 - \sum_{s^\prime} P_h^o(s^\prime \mid s,a) \exp\left(\frac{\eta -  V_{h+1}^{\pi_k}(s^\prime) - \lambda}{\lambda}\right) = 0\,, \\
    \eta &= \lambda - \lambda \log\left(\sum_{s^\prime} P_h^o(s^\prime \mid s,a)\exp\left(\frac{ -  V_{h+1}^{\pi_k}(s^\prime) }{\lambda}\right)\right) \,.
\end{align*}

Directly optimize over $\eta$,  we can reduce the problem to
\begin{align*}
    L(\lambda) (s,a) &= \lambda(1-\rho) -  \lambda \log\left(\sum_{s^\prime} P_h^o(s^\prime \mid s,a)\exp\left(\frac{ -  V_{h+1}^{\pi_k}(s^\prime) }{\lambda}\right)\right)  - \lambda \,, \\
    & = - \lambda \rho - \lambda \log\left(\sum_{s^\prime} P_h^o(s^\prime \mid s,a)\exp\left(\frac{ -  V_{h+1}^{\pi_k}(s^\prime) }{\lambda}\right)\right)   \,.
\end{align*}
Define $g(\lambda, P_h^o) = -L(\lambda) (s,a) $ as 
\begin{align*}
    g(\lambda, P_h^o) 
    = \ &  \lambda \rho + \lambda \log\left(\sum_{s^\prime} P_h^o(s^\prime \mid s,a)\exp\left(\frac{ -  V_{h+1}^{\pi_k}(s^\prime) }{\lambda}\right)\right)\,.
\end{align*}


Note that the Lagrangian multiplier $\lambda \ge 0$.
Then we prove $g$ is bounded within $[-H, H]$ over $[0, H/ \rho]$.

\begin{align*}
    g(\lambda, P_h^o) 
    = \ &  \lambda \rho + \lambda \log\left(\sum_{s^\prime} P_h^o(s^\prime \mid s,a)\exp\left(\frac{ -  V_{h+1}^{\pi_k}(s^\prime) }{\lambda}\right)\right)\,, \\
    \le \ &  \lambda \rho + \lambda \log\left(\sum_{s^\prime} P_h^o(s^\prime \mid s,a)\exp\left(\frac{ -  0 }{\lambda}\right)\right)\,, \\
    = \ & \lambda \rho \le H \,,
\end{align*}
where the first inequality follows from $V_{h+1}^{\pi_k}(s^\prime) \ge 0$ and the second inequality is by $\lambda \le H/ \rho$.

\begin{align*}
    g(\lambda, P_h^o) 
    = \ &  \lambda \rho + \lambda \log\left(\sum_{s^\prime} P_h^o(s^\prime \mid s,a)\exp\left(\frac{ -  V_{h+1}^{\pi_k}(s^\prime) }{\lambda}\right)\right)\,, \\
    \ge \ &  \lambda \rho + \lambda \log\left(\sum_{s^\prime} P_h^o(s^\prime \mid s,a)\exp\left(\frac{ -  H }{\lambda}\right)\right)\,, \\
    = \ & \lambda \rho - H  \ge -H \,,
\end{align*}
where the first inequality follows from $V_{h+1}^{\pi_k}(s^\prime) \le H$ and the second inequality is by $\lambda \ge 0$.

Moreover, from the induction above we know that for any $P$, $g(0, P) \le 0$ and for $\lambda > H / \rho$,
\begin{align*}
    g\left( \lambda , P \right) 
    \geq \lambda \rho  + \lambda \log (\exp( -H / \lambda)) > 0 \,. 
\end{align*}
Therefore, g achieves its minimum  over $\lambda \in [0, H / \rho]$. We remark that the same form is also used for sample complexity results ( \citep{badrinath2021robust,yang2021towards}).

We can now rewrite 
\begin{align*}
     \hPs \left(\hVpn\right)(s)  - \sPs\left(\hVpn\right)(s)
     = \ & \min_{0 \le \lambda_1 \le H/ \rho} g\left(\lambda_1, \hat{P}_h^{o,k}\right) - \min_{0 \le \lambda_2 \le H/ \rho} g\left(\lambda_2, P_h^o\right) \\
     \leq \ & \max_{0 \le \lambda \le H/ \rho} \left|g\left(\lambda, \hat{P}_h^{o,k}\right) -  g\left(\lambda, P_h^o\right) \right| \,.
\end{align*}

By \cite{nilim2005robust} (Appendix C), when $\lambda = 0$, $g\left(\lambda, \hat{P}_h^{o,k}\right) = g\left(\lambda, P_h^o\right) = \min_{s \in \gS} V^{\pi_k}_{h+1}(s)$. Therefore, it suffice to bound over $\max_{c \leq \lambda \le H/ \rho} \left|g\left(\lambda, \hat{P}_h^{o,k}\right) -  g\left(\lambda, P_h^o\right) \right|$, where $c > 0$. We now have 
\begin{align*}
    &  \left|g\left(\lambda, \hat{P}_h^{o,k}\right) -  g\left(\lambda, P_h^o\right) \right| \\
    = \ & \left| \lambda \log\left(\sum_{s^\prime} \hat{P}_h^{o,k}(s^\prime \mid s,a)\exp\left(\frac{ -  V_{h+1}^{\pi_k}(s^\prime) }{\lambda}\right)\right) - \lambda \log\left(\sum_{s^\prime} P_h^o(s^\prime \mid s,a)\exp\left(\frac{ -  V_{h+1}^{\pi_k}(s^\prime) }{\lambda}\right)\right)\right|\\
    = \ & \left | \lambda \log \left( 1 + \frac{\sum_{s^\prime} (\hat{P}_h^{o,k}(s^\prime \mid s,a) - P_h^o(s^\prime \mid s,a))\exp\left(\frac{ -  V_{h+1}^{\pi_k}(s^\prime) }{\lambda}\right)}{\sum_{s^\prime} P_h^o(s^\prime \mid s,a)\exp\left(\frac{ -  V_{h+1}^{\pi_k}(s^\prime) }{\lambda}\right)} \right)\right| \\
    \leq \ & 2 \lambda \left|  \frac{\sum_{s^\prime} (\hat{P}_h^{o,k}(s^\prime \mid s,a) - P_h^o(s^\prime \mid s,a))\exp\left(\frac{ -  V_{h+1}^{\pi_k}(s^\prime) }{\lambda}\right)}{\sum_{s^\prime} P_h^o(s^\prime \mid s,a)\exp\left(\frac{ -  V_{h+1}^{\pi_k}(s^\prime) }{\lambda}\right)} \right|\\
    \leq \ & 2\lambda \max_{s^\prime} \left| \frac{\hat{P}_h^{o,k}(s^\prime \mid s,a) - P_h^o(s^\prime \mid s,a)}{P_h^o(s^\prime \mid s,a)} \right|
\end{align*}
where the first inequality follows from $|\log(1 + x)| \leq 2|x|$ and the second inequality follows from the Holder's inequality.

By Hoeffding's inequality and an union bound over all $s,a^\prime$, $N_h^k(s,a)$, the following inequality holds with probability at least $1 - \delta$,
\begin{align*}
    \max_{s^\prime} \left| \hat{P}_h^{o,k}(s^\prime \mid s,a) - P_h^o(s^\prime \mid s,a) \right| \le \left\| \hat{P}_h^{o,k}(\cdot \mid s,a) - P_h^o(\cdot \mid s,a) \right\|_1  
    \leq \ &  \sqrt{\frac{4S \log(SAH^2 K / \delta)}{N_h^k(s,a)}} \,.
\end{align*}

Then we create an $\epsilon$-net $N_\epsilon(\lambda)$ with $g$ over $\lambda \in [0, H / \rho]$ such that
\begin{align*}
    \max_{\lambda \in [0, H / \rho]} |  g(\lambda, \hat{P}_h^{o,k})  - g(\lambda, P_h^o)| \leq \max_{\lambda\in N_\epsilon(\eta)} |  g(\lambda, \hat{P}_h^{o,k})  - g(\lambda, P_h^o)| + 2 \epsilon \,.
\end{align*}

Then we know that $| N_\epsilon(\lambda)| $ is bounded by the area of the rectangle $[0, H/ \rho] \times [-H, H] $ over $\epsilon^2$, 
\begin{align*}
     | N_\epsilon(\lambda)| \leq \frac{2H^2}{\rho \epsilon^2} \,.
\end{align*}


Taking an union bound over $N_\epsilon(\lambda)$ and denote $c = \min_{s^\prime} \limits P_h^o(\cdot \mid s,a)$, we have the following inequality happens with at least $1-\delta^\prime$ probability:
\begin{align*}
\hPs (\hVpn)(s)  - \sPs (\hVpn)(s) 
\leq \ & \max_{\lambda\in [0, H / \rho]} |  g(\lambda, \hat{P}_h^{o,k})  - g(\lambda, P_h^o)| \\
    \leq \ & \max_{\lambda\in N_\epsilon(\lambda)} |  g(\lambda, \hat{P}_h^{o,k})  - g(\lambda, P_h^o)| + 2 \epsilon \\
    \leq \ & 2 \frac{H}{\rho}  \max_{s^\prime} \left| \frac{\hat{P}_h^{o,k}(s^\prime \mid s,a) - P_h^o(s^\prime \mid s,a)}{P_h^o(s^\prime \mid s,a)} \right| + 2 \epsilon \\
    \le \ &  2 \frac{H}{\rho c} \sqrt{\frac{4S \log(2SAH^4 K / \delta^\prime \rho \epsilon^2)}{N_h^k(s,a)}}  + 2 \epsilon \,,
\end{align*}

Take $\epsilon = \frac{1}{2\sqrt{K}}$, then 
\begin{align*}
    \hPs (\hVpn)(s)  - \sPs (\hVpn)(s) 
    \leq \ & 2 \frac{H}{\rho c}  \sqrt{\frac{4 S  \log(8SAH^4K^2/ \delta^\prime \rho)}{N_h^k(s,a)}} +  \frac{1}{\sqrt{K}} \,.
\end{align*}
\end{proof}

\newpage
\section{Proof of Proposition 1}\label{appendix:prop}
\hard*
\begin{proof}
We consider a robust MDP with three states $s_0, s_1, s_2$ and two actions $a_0, a_1$. Without loss of generality, we let $s_0$ be the initial state.  On the initial state $s_0$, both actions will lead to a reward of $0$. On state $s_1$, a reward of $1 / (H-1)$ is given for both actions. On state $s_2$, a reward of $-1 / (H-1)$ is given for both actions. The nominal transition dynamic of the MDP is the following. Taking action $a_0$ on $s_0$ will be transited to $s_1$ with a probability of $\epsilon$ and be transited to $s_2$ with a probability of $\epsilon$, while $\epsilon > 0.5$. Taking the other action $a_1$ will have equal probability of transiting to $s_1$ and $s_2$. The states $s_1$ and $s_2$ are absorbing, in the sense that taking any action on these two states will be transited by to the same state. The transition of the MDP is also illustrated in Figure \ref{fig:hard}, where a dashed line denotes a probabilistic transition and a solid line denotes deterministic transition.  
\begin{figure}[h]
    \centering
    \includegraphics[width=0.4\textwidth]{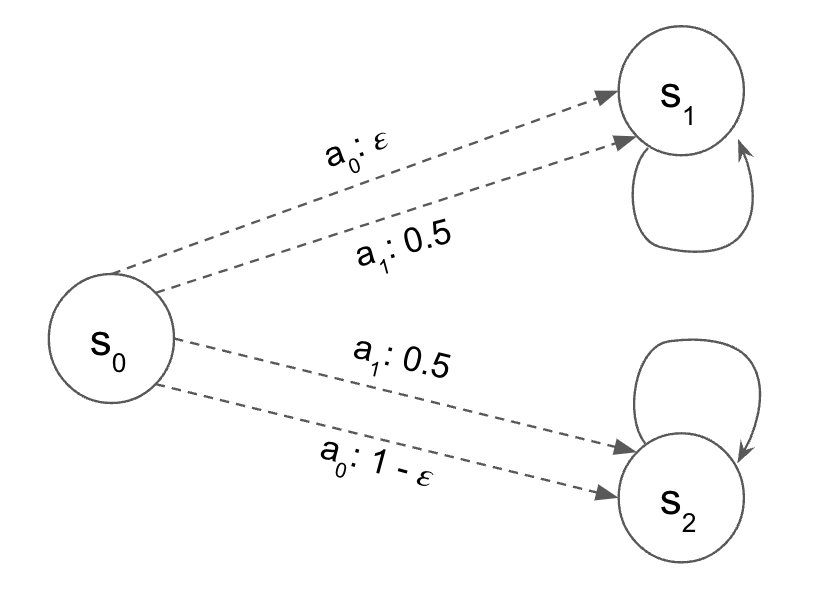}
    \includegraphics[width=0.4\textwidth]{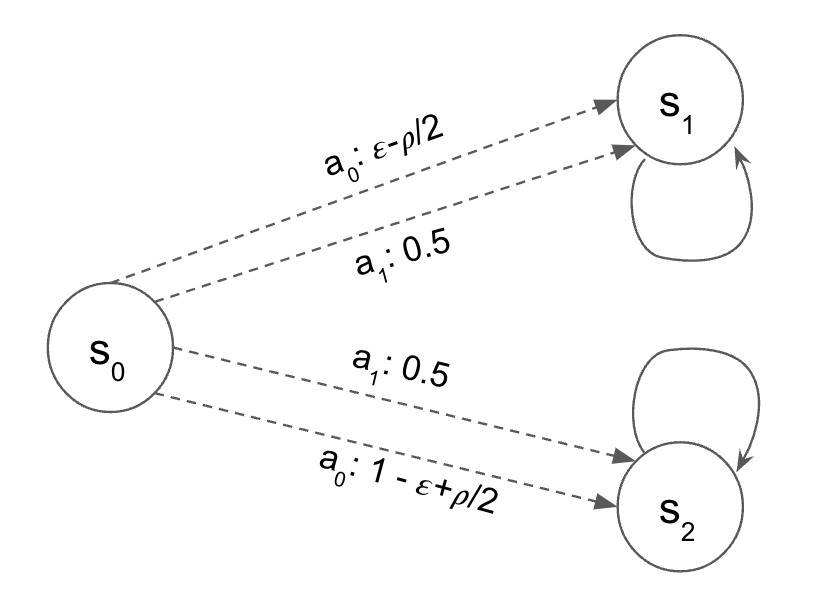}
    \caption{The left figure describes the nominal transition dynamic of the MDP. The right figure describes the robust transition dynamic of the MDP. }
    \label{fig:hard}
\end{figure}
With the nominal transition, it is clear that an optimal policy would be always taking $a_0$. Denote this policy as $\pi_{o, \ast}$, the value for this policy under nominal transition over $K$ episodes is 
\begin{align*}
    V^{\pi_{o, \ast}}(s_0) = K (H - 1) \left(\epsilon \cdot \frac{1}{H - 1} - (1 - \epsilon) \cdot \frac{1}{H - 1} \right) = 2\epsilon - 1 > 0 \,,
\end{align*}
where the last inequality is due to $\epsilon > 0$.

However, consider the uncertainty radius $\rho$ and the robust transition denoted by the right figure of Figure \ref{fig:hard}. That is, taking $a_0$ on $s_0$ will leads to a transition to $s_1$ with probability $ \epsilon - \rho / 2$ and to $s_2$ with probability $1 - \epsilon + \rho / 2$. Note that as $\epsilon > 0.5$, $\rho \leq 1$, $\epsilon - \rho / 2 > 0$. Moreover, this transition is indeed the worst case transition for any non-uniform policy. Let $\Tilde{V}$ denotes the robust value under the above described transition. With a uniform policy $\pi$, the value of it under this transition is 
\begin{align*}
    \Tilde{V}^\pi(s_0) = K (H - 1) \left(0.5 \left(\epsilon - \frac{\rho}{2} \right)\cdot \frac{1}{H - 1} - 0.5 \left(1 - \epsilon +\frac{\rho}{2}\right)) \cdot \frac{1}{H - 1} \right) = \epsilon - \rho / 2 - 0.5 \,.
\end{align*}
The value of $\pi_{o,\ast}$ is, however, 
\begin{align*}
    \Tilde{V}^{\pi_{o, \ast}}(s_0) = K (H - 1) \left(\left(\epsilon - \frac{\rho}{2} \right)\cdot \frac{1}{H - 1} - \left(1 - \epsilon +\frac{\rho}{2}\right)) \cdot \frac{1}{H - 1} \right) = 2\epsilon - \rho  - 1 \,.
\end{align*}
For any $2\epsilon - 1 \leq \rho \leq 1$, we have $\Tilde{V}^{\pi_{o, \ast}}(s_0) \leq \Tilde{V}^\pi(s_0) $. Since $\epsilon > 0.5$ is arbitrary, the optimal policy under the nominal transition is non-robust even under the slightest perturbation. 
\end{proof}
\newpage
\section{Auxiliary lemmas}
\begin{lem}[\cite{bartlett2013theoretical}]\label{lem:eps_cover}
An $\epsilon$-cover of a subset $T$ of a pseudometric space $(S, d)$ is a set $\hat{T} \subset T$ such that for each $t \in T$ there is a $\hat{t} \in \hat{T}$ such that $d(t, \hat{t}) \leq \epsilon$. The $\epsilon$-covering number of $T$ is
$$
N(\epsilon, T, d)=\min \left\{|\hat{T}|: \hat{T} \text { is an } \epsilon \text {-cover of } T \right\} \,.
$$
    Let $F_{d}$ be the set of $L$-Lipschitz functions (wrt $\|\cdot\|_{\infty}$ ) mapping from $[0,1]^{d}$ to $[0,1]$. Then
$$
\log N\left(\epsilon, F_{d},\|\cdot\|_{\infty}\right)=\Theta\left(\left(\frac{L}{\epsilon}\right)^{d}\right) \,.
$$
\end{lem}


\begin{lem}[Lemma 7.5 \cite{agarwal2019reinforcement}]\label{lem:visit}
For arbitrary $K$ sequence of trajectories $\{s_h^k, a_h^k\}_{h=1}^H$, $k = 1, \ldots, K$, we have
    \begin{align*}
        \sum^K_{k=1} \sum^H_{h=1} \frac{1}{\sqrt{N_h^k(s_h^k, a_h^k)}} \leq 2 H \sqrt{SAK} \,.
    \end{align*}
\end{lem}
\begin{proof}
We have
\begin{align*}
\sum^K_{k=1} \sum^H_{h=1} \frac{1}{\sqrt{N_h^k\left(s_h^k, a_h^k\right)}} 
= \ & \sum_{h=1}^{H} \sum_{(s, a) \in \gS \times \gA} \sum_{i=1}^{N_h^K(s, a)} \frac{1}{\sqrt{i}} \\
\leq \ & 2 \sum_{h=1}^{H}  \sum_{(s, a) \in \gS \times \gA}\sqrt{N_h^K(s, a)} \\
\leq \ & \sum_{h=1}^{H} \sqrt{S A \sum_{s, a} N_h^K(s, a)} \\
=\ & H \sqrt{S A K} \,,
\end{align*}
where the first inequality is by $\sum^N_{i=1} \frac{1}{ \sqrt{i}} \leq 2 \sqrt{N}$ and the second inequality follows by Cauchy-Schwarz inequality.
\end{proof}

\begin{lem}[Fundamental inequality of Online Mirror Descent for RL (Lemma 17 \cite{shani2020optimistic})]\label{lem:omd}
Let $\beta>0$. Let $\pi_{h}^{1}(\cdot \mid s)$ be the uniform distribution. Then, by updating with OMD and with KL divergence regularization, for any $k \in[K], h \in[H]$ and $s \in \mathcal{S}$, the following holds for any stationary policy $\pi$, 
\begin{align}\label{eq:omd}
    \sum_{k=1}^{K}\left\langle Q_{h}^{k}(\cdot \mid s), \pi_{h}^{k}(\cdot \mid s)-\pi_{h}(\cdot \mid s)\right\rangle \leq \frac{\log A}{\beta}+\frac{\beta}{2} \sum_{k=1}^{K} \sum_{a} \pi_{h}^{k}(a \mid s)\left(Q_{h}^{k}(s, a)\right)^{2} \,.
\end{align}
\end{lem}
\newpage
\section{More experimental details}
\paragraph{Other configurations and set up }
The episode length is set to $20$ and all algorithms are trained with $3000$ episodes. The evaluation results are averaged over $20$ runs and is presented with $1$ standard deviation. All experiments are conducted with 64 core ADM 3990X. 
\paragraph{Results with KL divergence uncertainty sets}
With the uncertainty set described with KL divergence, we present the following experimental results. All other configurations and set up remains the same with those for uncertainty set with $\ell_1$ distance.
\begin{figure}[h]
     \centering
     \begin{subfigure}[b]{0.3\textwidth}
         \centering
         \includegraphics[width=\textwidth]{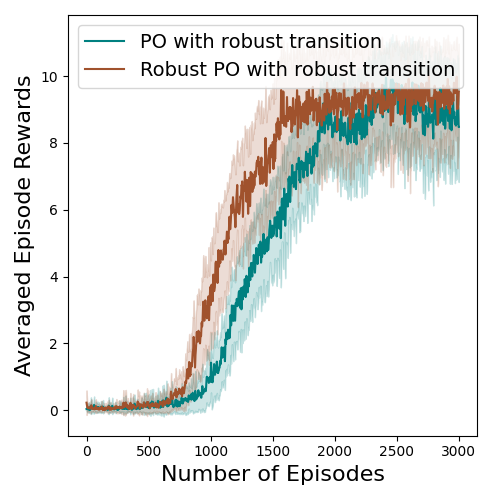}
         \caption{$\rho = 0.1$}
         \label{fig:kl_rho1}
     \end{subfigure}
     \hfill
     \begin{subfigure}[b]{0.3\textwidth}
         \centering
         \includegraphics[width=\textwidth]{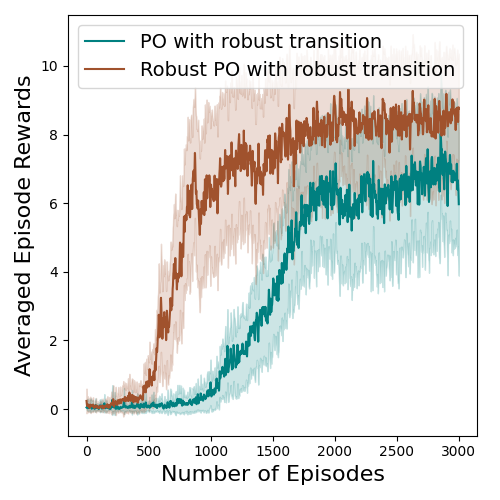}
          \caption{$\rho = 0.2$}
         \label{fig:kl_rho2}
     \end{subfigure}
     \hfill
     \begin{subfigure}[b]{0.3\textwidth}
         \centering
         \includegraphics[width=\textwidth]{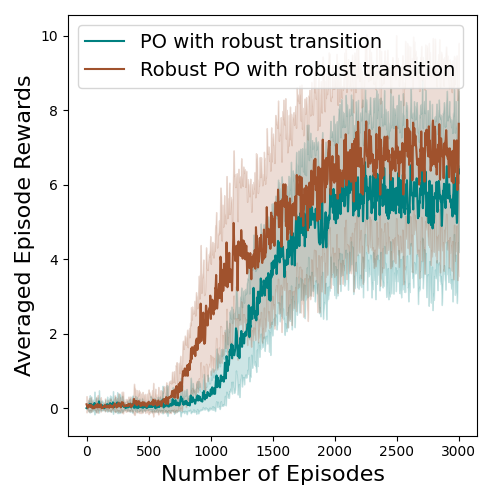}
          \caption{$\rho = 0.3$}         \label{fig:kl_rho3}
     \end{subfigure}
    \caption{Cumulative rewards obtained by robust and non-robust policy optimization on robust transition with different level of uncertainty $\rho = 0.1, 0.2, 0.3$ under KL divergence.}
        \label{fig:exp_kl}
\end{figure}

\end{document}